\documentclass[conference]{IEEEtran}

\usepackage{fancyhdr}
\fancypagestyle{preprint}{
  \fancyhf{}
  \fancyhead[L]{PREPRINT VERSION}
  
  \setlength{\topmargin}{-0.7in}
  \setlength{\headheight}{15pt}
  \setlength{\headsep}{25pt}
}

\usepackage{times}
\usepackage{algorithm}
\usepackage{algpseudocode}
\usepackage[numbers]{natbib}
\usepackage{multicol}
\usepackage[bookmarks=true,hidelinks]{hyperref}
\hypersetup{hidelinks}
\usepackage{graphicx}
\usepackage{booktabs}
\usepackage{amsmath}
\usepackage{amsfonts}
\usepackage{amssymb}
\usepackage{amsthm}
\usepackage{cuted}
\usepackage{capt-of}
\usepackage{dblfloatfix}
\usepackage{multirow}
\usepackage{subcaption}
\usepackage{tabularx}
\usepackage{siunitx}
\usepackage{makecell}

\usepackage{arydshln}
\usepackage{xcolor}
\usepackage{comment}
\usepackage{array}

\newtheorem*{remark}{Remark}
\newtheorem{theorem}{Theorem}
\newtheorem{lemma}{Lemma}

\begin{document}
\title{
Natural Functional Gradients \\ for Smooth Trajectory Optimization
}

\author{
\IEEEauthorblockN{Kisang Park$^{1}$, Chanwoo Kim$^{1}$, Kyungjae Lee$^{2}$, Sungjoon Choi$^{1,3}$}
\IEEEauthorblockA{$^{1}$Department of Artificial Intelligence, Korea University, Seoul, Republic of Korea}
\IEEEauthorblockA{$^{2}$Department of Statistics, Korea University, Seoul, Republic of Korea}
\IEEEauthorblockA{$^{3}$RLWRLD, Seoul, Republic of Korea}
\IEEEauthorblockA{\texttt{\{qkrrltkd1626, chanwoo-kim, kyungjae\_lee, sungjoon-choi\}@korea.ac.kr}}
}

\maketitle
\pagestyle{preprint}
\IEEEpeerreviewmaketitle

\begin{abstract}
Generating collision-free and smooth motions remains a central challenge in robotic manipulation, particularly in cluttered environments and narrow passages where feasible regions are highly constrained and fragmented. 
We propose a trajectory optimization framework that performs geometry-aware updates directly in function space using natural functional gradients. 
The method optimizes a Gaussian-smoothed surrogate objective that regularizes the optimization landscape through smooth trajectory perturbations while preserving trajectory-level structure. 
Because the updates are defined intrinsically in function space, trajectory regularity can be controlled independently of a particular time discretization. 
We derive a practical Monte-Carlo estimator of the natural functional gradient that requires only black-box trajectory evaluations, making the method applicable when analytic gradients are unavailable or unreliable due to collision checking and contact-rich simulation. 
Experiments on constrained robotic manipulation tasks demonstrate that the proposed method improves trajectory feasibility and produces smoother motions than representative planning and trajectory optimization baselines in environments with narrow geometric clearances. 
Additional results, videos, and implementation details are available at
\href{https://kisangpark.github.io/natural-functional-gradient/}{the project page.}
\end{abstract}

\section{Introduction}

Trajectory optimization is a fundamental component of robotic motion generation. 
In manipulation tasks, a planner must simultaneously satisfy collision avoidance, kinematic feasibility, and trajectory smoothness under limited geometric clearance. 
These requirements become particularly challenging in cluttered environments, where feasible motions occupy only a small subset of the trajectory space and small deviations can induce collisions. 
In such regimes, feasibility alone is insufficient: the temporal regularity of the solution directly impacts tracking accuracy and robustness during execution.

Existing approaches provide complementary advantages but expose a persistent trade-off between feasibility and smoothness. 
Sampling-based planners effectively explore highly nonconvex configuration spaces and discover collision-free solutions, but often return piecewise-linear paths that require additional smoothing before execution \cite{kavraki1996prm,lavalle1998rrt,karaman2011rrtstar}. 
Trajectory optimization methods instead generate smooth motions directly, but they typically rely on local information and can become sensitive to initialization in fragmented feasible regions \cite{zucker2013chomp,schulman2014trajopt,kalakrishnan2011stomp}. 
Recent GPU-accelerated methods improve optimization throughput, yet trajectory generation under narrow geometric constraints remains challenging \cite{sundaralingam2023curobo}.

A central difficulty is that feasibility and trajectory regularity are often handled using separate representations and optimization procedures. 
Methods focused on feasibility commonly operate on sparse geometric paths and defer smoothness to post-processing, while smooth trajectory optimization methods directly optimize discretized trajectories and can struggle to escape poor local regions. 
As a result, maintaining trajectory stability while satisfying collision constraints remains difficult, particularly when optimization relies on black-box simulation costs or lacks reliable analytic gradients.

This paper adopts a function-space perspective on trajectory optimization. 
Rather than representing trajectories using a fixed finite-dimensional parameterization, we model trajectories as elements of a Hilbert space and perform updates directly in function space. 
This formulation enables trajectory regularity to be controlled through the geometry induced by Gaussian kernels and allows smooth perturbations to be generated independently of a particular trajectory discretization \cite{marinho2016rkhs,mukadam2016gpmp,dong2016gpmp2}. 
Building on this viewpoint, we propose a zeroth-order trajectory optimization framework based on natural functional gradients induced by Gaussian smoothing.

The proposed update direction admits a simple weight-times-noise form, enabling practical Monte-Carlo optimization even when analytic gradients are unavailable or unreliable due to collision checking and contact-rich simulation. 
We additionally analyze the Gaussian-smoothed objective directly in Hilbert space and establish convergence guarantees without defining the optimization geometry through a particular finite-dimensional trajectory parameterization. 
The resulting framework combines smooth function-space perturbations with gradient-free optimization for constrained robotic motion generation tasks.

The contributions of this paper are threefold. 
First, we introduce a function-space trajectory optimization framework based on natural functional gradients that performs geometry-aware trajectory updates directly in Hilbert space. 
Second, we develop a Gaussian smoothing formulation for zeroth-order trajectory optimization and provide convergence guarantees in function space. 
Third, we demonstrate on constrained robotic manipulation tasks that the proposed method improves trajectory feasibility and produces smoother trajectories than representative planning and trajectory optimization baselines in cluttered environments with narrow geometric clearances.

\begin{table*}[t]
\centering
\footnotesize
\setlength{\tabcolsep}{5pt}
\renewcommand{\arraystretch}{1.15}

\caption{Conceptual comparisons of representative motion planning and trajectory optimization approaches.}
\label{tab:related_work_comparison}
\vspace{3pt}

\resizebox{\textwidth}{!}{
\begin{tabular}{lcccc}
\toprule
Method family
& Trajectory representation
& Cost differentiability required
& Gradient access
& Smoothness handling \\
\midrule

Sampling-based
{\cite{kavraki1996prm,lavalle1998rrt}}
& Discrete path
& No
& None
& Post-processing \\

Gradient-based
{\cite{zucker2013chomp,schulman2014trajopt}}
& Finite-dimensional
& Yes / approximated
& Analytic
& Explicit regularization \\

Stochastic optimization
{\cite{kalakrishnan2011stomp,theodorou2010pi2}}
& Finite-dimensional
& No
& Zeroth-order
& Implicit / penalty-based \\

Inference-based / GP planning
{\cite{toussaint2009approxinf,levine2018controlinf, mukadam2016gpmp}}
& Probabilistic trajectory model
& Yes / model-dependent
& Variational / analytic
& Model-based prior \\

Evolutionary algorithms
{\cite{hansen2001completely}}
& Finite-dimensional
& No
& Zeroth-order
& Not intrinsic \\

Zeroth-order optimization
{\cite{flaxman2005online,duchi2015optimal}}
& Finite-dimensional
& No
& Zeroth-order
& Via smoothing \\

Homotopy / continuation
{\cite{mobahi2015continuation,hazan2016graduated}}
& Varies
& Varies
& Varies
& Via smoothing schedule \\

\midrule

Proposed (natural functional gradients)
& Function space
& No
& Zeroth-order
& Intrinsic (geometry-induced) \\

\bottomrule
\end{tabular}
}

\vspace{-8pt}
\end{table*}

\section{Related Work}
\label{sec:related_work}
Trajectory optimization for robotic motion generation has been studied from multiple complementary perspectives, including sampling-based planning, gradient-based trajectory optimization, stochastic optimization, and zeroth-order black-box optimization. 
These approaches address different aspects of the core challenge: discovering feasible motions in highly nonconvex environments while maintaining trajectory smoothness and execution stability. 
Table~\ref{tab:related_work_comparison} summarizes representative approaches according to trajectory representation, gradient access, differentiability requirements, and smoothness handling.

Sampling-based planners such as PRM \cite{kavraki1996prm} and RRT \cite{lavalle1998rrt,kuffner2000rrtconnect,karaman2011rrtstar} explore configuration spaces without relying on gradient information and are effective for discovering collision-free motions in cluttered environments. 
However, the resulting solutions are typically represented as sparse geometric paths and often require additional shortcutting or smoothing before execution. 
Gradient-based trajectory optimization methods such as CHOMP and TrajOpt \cite{ratliff2009chomp,zucker2013chomp,schulman2013trajopt,schulman2014trajopt} instead optimize smooth trajectories directly using local cost information. 
Related probabilistic and function-space formulations incorporate smoothness priors through Gaussian-process trajectory models and reproducing kernel Hilbert spaces \cite{mukadam2016gpmp,dong2016gpmp2,marinho2016rkhs}. 
Despite their effectiveness, these methods commonly rely on analytic or approximated gradients and can become sensitive to initialization in environments dominated by narrow feasible regions or nondifferentiable collision costs.

Stochastic and gradient-free approaches relax differentiability requirements by estimating updates from sampled perturbations or rollout trajectories. 
Representative examples include STOMP \cite{kalakrishnan2011stomp}, path-integral methods \cite{theodorou2010pi2,williams2015mppi}, evolutionary strategies \cite{hansen2001completely,hansen2019pycma}, and recent GPU-parallelized motion generation systems \cite{sundaralingam2023curobo}. 
These methods are generally more flexible with respect to black-box objectives and simulation-based costs, but most operate on finite-dimensional trajectory parameterizations and enforce smoothness indirectly through discretization-dependent penalties or post-processing. 
More broadly, zeroth-order optimization theory and Gaussian smoothing methods provide principled tools for optimization using only function evaluations \cite{flaxman2005online,nesterov2017random,duchi2015optimal,gao2022generalizing}. 
Related continuation and homotopy methods further study how smoothed objectives can improve optimization behavior in nonconvex landscapes \cite{mobahi2015continuation,hazan2016graduated}.

Our work combines these perspectives through a function-space formulation of zeroth-order trajectory optimization. 
Rather than optimizing a finite-dimensional trajectory parameterization, we perform geometry-aware updates directly in Hilbert space using natural functional gradients induced by Gaussian smoothing. 
This formulation enables smooth trajectory perturbations to be generated intrinsically in function space while remaining compatible with black-box trajectory evaluation. 
Unlike prior smooth trajectory optimization methods, the proposed approach does not rely on analytic gradients or discretization-tuned smoothness penalties. 
At the same time, it differs from conventional stochastic optimization methods by explicitly incorporating function-space geometry into the update rule and smoothing process.

We additionally analyze the Gaussian-smoothed objective directly in function space and establish convergence guarantees without reducing the optimization problem to a finite-dimensional representation. 
As a result, the proposed method occupies a distinct position at the intersection of function-space trajectory optimization, Gaussian smoothing, and gradient-free robotic motion generation.

\section{Preliminaries}
\label{sec:preliminaries}

This section introduces the minimum notation and background required to describe our method.
Our goal is to formalize trajectories as elements of a function space and to fix the probabilistic and geometric structures used later for defining geometry-aware, zeroth-order updates.
All methodological contributions begin in Section~\ref{sec:method}; the material here is standard and is included only to establish a common framework.

\subsection{Trajectories in a Hilbert space}
Let $T>0$ be a fixed time horizon and let $\xi:[0,T]\to\mathbb R^d$ denote a continuous-time trajectory.
We model trajectories as elements of a separable real Hilbert space $\mathcal H$ equipped with inner product $\langle\cdot,\cdot\rangle_{\mathcal H}$ and norm $\|\cdot\|_{\mathcal H}$.
This viewpoint is deliberately resolution-independent: time discretization is used only for numerical approximation and does not define the optimization geometry.
In particular, the choice of $\mathcal H$ encodes a notion of regularity (e.g., smoothness) at the level of functions rather than finite-dimensional vectors.

In robotic motion generation, this abstraction is natural because physically executable trajectories must satisfy temporal regularity constraints, and many objectives are most naturally expressed as functionals of the entire path.
Sobolev spaces and reproducing kernel Hilbert spaces (RKHSs) are commonly used to encode smoothness and prior structure in trajectory optimization and motion planning
\cite{marinho2016rkhs,mukadam2016gpmp,dong2016gpmp2}.
Although our theoretical development is stated for a general Hilbert space, later sections will specialize to an RKHS-based implementation where the induced geometry directly shapes the smoothness of updates.

\subsection{Zeroth-order objective}
Let $f:\mathcal H\to\mathbb R\cup\{-\infty\}$ be a trajectory objective functional.
In robotic manipulation and motion planning, $f$ may aggregate task objectives, control effort, collision penalties, and environment-specific costs.
We do not assume that $f$ is differentiable with respect to $\xi$, nor that analytic gradients are available or reliable.
Instead, we assume only zeroth-order access: given a candidate trajectory $\xi$, the value $f(\xi)$ can be evaluated, for example by collision checking and simulation-based rollout.

This assumption reflects common practice in robotics and simulation-based control.
Collision checking and contact-rich dynamics often introduce nondifferentiabilities, and many pipelines incorporate black-box components or hard feasibility checks.
As a result, gradient-free and zeroth-order optimization methods are widely used when analytic derivatives are unavailable or unstable
\cite{flaxman2005online,duchi2015optimal,nesterov2017random}.
Our method builds on this zeroth-order setting while explicitly addressing the need for smooth, executable trajectories.

\subsection{Gaussian perturbations and Cameron--Martin space}
Let $\Sigma:\mathcal H\to\mathcal H$ be a self-adjoint, positive, trace-class, and injective operator, and let $\sigma>0$.
We consider Gaussian perturbations $\varepsilon\sim\mathcal N(0,\sigma^2\Sigma)$ taking values in $\mathcal H$.
Such perturbations are the basis of Gaussian smoothing, a classical technique for regularizing nonsmooth or nonconvex objectives by averaging over local perturbations
\cite{nesterov2017random,gao2022generalizing}.
In the context of trajectories, the covariance operator $\Sigma$ shapes how perturbations vary over time and across dimensions, thereby directly influencing the smoothness of induced updates.

A distinctive feature of Gaussian measures in infinite-dimensional spaces is that differentiation is naturally restricted to a subspace determined by the covariance.
The associated Cameron--Martin space is defined as $\mathcal H_\Sigma=\mathrm{Range}(\Sigma^{1/2})$, equipped with the inner product
\begin{equation}
\langle h_1,h_2\rangle_{\mathcal H_\Sigma}
=
\langle \Sigma^{-1/2}h_1,\Sigma^{-1/2}h_2\rangle_{\mathcal H},
\quad h_1,h_2\in\mathcal H_\Sigma.
\end{equation}
The Cameron--Martin construction characterizes precisely the directions along which shifts of a Gaussian measure remain absolutely continuous and thus differentiable
\cite{amari1998natural, bogachev1998gaussian, da2014stochastic}.
In Section~\ref{sec:method}, this geometry will induce a natural functional gradient for Gaussian-smoothed objectives, yielding smooth trajectory updates aligned with the covariance structure.

\section{Trajectory Optimization via Natural Functional Gradient}
\label{sec:method}

We now present our trajectory optimization method based on natural functional gradients. 
We treat a trajectory as a time-indexed function (e.g., $t \mapsto q(t)$ or $t \mapsto x(t)$), rather than a finite vector of waypoints, and formulate the optimization in a Hilbert space where the inner product encodes trajectory smoothness.
In this setting, the kernel defines the geometry of the trajectory space, determining how trajectories are compared and how updates are shaped.

The central idea is to perform zeroth-order optimization directly in this function space, rather than on a finite-dimensional parameterization.
By combining exponential transformation with Gaussian smoothing in a Hilbert space, we obtain a geometry-aware update rule that produces smooth trajectories intrinsically while requiring only black-box evaluations of the underlying cost functional.
This resulting update can be interpreted as a smooth deformation of the entire trajectory.

Throughout this section, we emphasize a clear separation between three conceptual components:
(i) the definition of a Gaussian-smoothed surrogate objective in function space,
(ii) the derivation of its natural functional gradient under the Cameron--Martin geometry, and
(iii) a practical discretization that enables efficient Monte-Carlo estimation and implementation.
This organization mirrors the logical flow from problem formulation to theory and finally to computation.

\subsection{Gaussian-Smoothed Surrogate Objective}

We consider trajectory optimization over a feasible set $S \subset \mathcal H$, where $\mathcal H$ is a separable real Hilbert space of trajectories $\xi:[0,T]\to\mathbb R^d$.
The feasible set $S$ encodes boundary conditions and admissibility constraints such as collision avoidance and joint limits.
We assume only zeroth-order access to an objective functional $f:S\to\mathbb R$, evaluated through collision checking and/or simulation.

To handle feasibility within a unified objective, we extend $f$ to all of $\mathcal H$ by setting $f(\xi)=-\infty$ for $\xi\notin S$.
This allows us to write all subsequent objectives as expectations over $\mathcal H$ without introducing explicit constraints.
For a power parameter $N_{\mathrm{pow}}>0$, we define the exponentially transformed functional
{\small\begin{equation}
g_{N_{\mathrm{pow}}}(\xi)
:=
\exp\!\big(N_{\mathrm{pow}}\, f(\xi)\big),
\label{eq:g_def}
\end{equation}}
which amplifies relative differences between trajectories while enforcing feasibility via the indicator.

We smooth $g_{N_{\mathrm{pow}}}$ by Gaussian perturbations in $\mathcal H$.
Let $\Sigma:\mathcal H\to\mathcal H$ be a self-adjoint, positive, trace-class, and injective covariance operator, and let
\begin{equation}
\varepsilon \sim \mathcal N(0,\sigma^2\Sigma),
\label{eq:eps_def}
\end{equation}
where $\sigma>0$ controls the magnitude of perturbations.
The Gaussian-smoothed surrogate objective is then defined as
\begin{equation}
F_{N_{\mathrm{pow}}}(\mu)
:=
\mathbb E\!\left[g_{N_{\mathrm{pow}}}(\mu+\varepsilon)\right],
\qquad \mu\in\mathcal H,
\label{eq:FN_def}
\end{equation}
where $\mu$ denotes the current mean trajectory.

The role of this construction is twofold.
First, Gaussian smoothing regularizes the optimization landscape by averaging over local perturbations in trajectory space.
Second, the exponential transformation induces a homotopy-like effect: as $N_{\mathrm{pow}}$ increases, the surrogate objective increasingly concentrates around trajectories with high original cost.
Crucially, evaluating $F_{N_{\mathrm{pow}}}$ requires only evaluations of $f(\mu+\varepsilon)$ and does not rely on analytic gradients.

\vspace{-5pt}

\subsection{Natural Functional Gradient}
{\begin{algorithm}[t]
\caption{Natural Functional Gradient}
\label{alg:nfg_rkhs}
\begin{algorithmic}[1]
\Require
Initial mean trajectory $\mu_0$ (discretized on grid $\{t_i\}_{i=1}^m$);
trajectory objective functional $f(\cdot)$ (black-box evaluation);
kernel $k(\cdot,\cdot)$ and regularization $\lambda>0$;
noise scale $\sigma>0$;
power parameter $N_{\mathrm{pow}}>0$;
Monte-Carlo samples $B$;
stepsizes $\{\eta_k\}_{k=0}^{K-1}$;
iterations $K$.
\Ensure
Optimized mean trajectory $\mu_K$.

\State Construct kernel matrix $K\in\mathbb R^{m\times m}$ by $K_{ij}=k(t_i,t_j)$ and set $K_\lambda \gets K+\lambda I$.
\State Compute a matrix square root $L$ such that $LL^\top=K_\lambda$ (e.g., Cholesky factor).
\State Initialize $\mu \gets \mu_0$.
\For{$k=0$ to $K-1$}
    \State $\widehat g \gets 0 \in \mathbb R^{m}$.
    \For{$s=1$ to $B$}
        \State Sample $z^{(s)} \sim \mathcal N(0,I_m)$.
        \State $\varepsilon^{(s)} \gets \sigma L z^{(s)}$ \Comment{$\varepsilon^{(s)}\sim\mathcal N(0,\sigma^2K_\lambda)$}
        \State Evaluate weight $w^{(s)} \gets \exp\!\big(N_{\mathrm{pow}}\, f(\mu+\varepsilon^{(s)})\big)$.
        \State $\widehat g \gets \widehat g + \frac{1}{B\sigma^{2}}\,w^{(s)}\,\varepsilon^{(s)}$.
    \EndFor
    \State $\mu \gets \mu + \eta_k\,\widehat g$.
\EndFor
\State \Return $\mu$.
\end{algorithmic}
\end{algorithm}
}

Gaussian measures in infinite-dimensional spaces admit differentiation only along a covariance-dependent subspace.
The Gaussian measure $\mathcal N(0,\sigma^2\Sigma)$ induces the Cameron--Martin space
\begin{equation}
\mathcal H_\Sigma := \mathrm{Range}(\Sigma^{1/2}),
\end{equation}
equipped with the inner product
\begin{equation}
\langle h_1,h_2\rangle_{\mathcal H_\Sigma}
=
\langle \Sigma^{-1/2}h_1,\Sigma^{-1/2}h_2\rangle_{\mathcal H},
\qquad h_1,h_2\in\mathcal H_\Sigma.
\label{eq:cm_ip}
\end{equation}
This geometry characterizes the admissible directions along which the smoothed objective $F_{N_{\mathrm{pow}}}$ can be differentiated.

\begin{lemma}[Directional derivative identity]
\label{lem:cm_directional}
For any $\mu\in\mathcal H$ and any $h\in\mathcal H_\Sigma$, the directional derivative of $F_{N_{\mathrm{pow}}}$ exists and satisfies
\begin{equation}
D_hF_{N_{\mathrm{pow}}}(\mu)
=
\mathbb E\!\left[
g_{N_{\mathrm{pow}}}(\mu+\varepsilon)
\sum_{k:\lambda_k>0}
\frac{
\langle h,e_k\rangle_{\mathcal H}
}{
\sigma\sqrt{\lambda_k}
}
\xi_k
\right],
\label{eq:Dh_identity}
\end{equation}
where $\varepsilon=\sigma\sum_{k:\lambda_k>0}\sqrt{\lambda_k}\,\xi_k e_k$ with $\xi_k\overset{i.i.d.}{\sim}\mathcal N(0,1)$ is the Karhunen--Lo\`eve representation associated with the eigendecomposition \(\Sigma e_k=\lambda_k e_k\).
\end{lemma}
The full proof can be found in the supplementary materials.
The result follows by rewriting $F_{N_{\mathrm{pow}}}(\mu+th)$ as an expectation with respect to a translated Gaussian measure and applying the Cameron--Martin theorem to obtain the corresponding Radon--Nikodym derivative (see Chapter~2 in \cite{bogachev1998gaussian}).

\begin{lemma}[Closed form of the natural functional gradient]
\label{lem:natgrad_closed}
There exists a unique $\nabla_{\mathcal H_\Sigma}F_{N_{\mathrm{pow}}}(\mu)\in\mathcal H_\Sigma$ such that
\[
D_hF_{N_{\mathrm{pow}}}(\mu)
=
\langle \nabla_{\mathcal H_\Sigma}F_{N_{\mathrm{pow}}}(\mu),h\rangle_{\mathcal H_\Sigma}
\quad\text{for all }h\in\mathcal H_\Sigma,
\]
and it admits the closed-form representation
\begin{equation}
\nabla_{\mathcal H_\Sigma}F_{N_{\mathrm{pow}}}(\mu)
=
\sigma^{-2}\,
\mathbb E\!\left[g_{N_{\mathrm{pow}}}(\mu+\varepsilon)\,\varepsilon\right].
\label{eq:natgrad_closed}
\end{equation}
\end{lemma}
Lemma~\ref{lem:natgrad_closed} shows that the directional derivative of $F_{N_{\mathrm{pow}}}$ admits a Riesz representer in the Cameron--Martin space, yielding the closed-form expression in \eqref{eq:natgrad_closed}~\cite{riesz2012functional}.
This representation leads to a practical estimator, where Gaussian perturbations are weighted by their exponentiated objective values.
Crucially, the resulting update depends only on zeroth-order evaluations of $f$, making it applicable in nondifferentiable or simulation-based settings.

\subsection{Monte-Carlo Estimation and Update Rule}

In practice, the expectation in \eqref{eq:natgrad_closed} is approximated using Monte-Carlo sampling.
Given i.i.d.\ samples $\varepsilon^{(1)},\dots,\varepsilon^{(B)} \sim \mathcal N(0,\sigma^2\Sigma)$, we use the unbiased estimator
\begin{equation}
\widehat{\nabla}_{\mathcal H_\Sigma}F_{N_{\mathrm{pow}}}(\mu)
=
\frac{1}{B\sigma^2}
\sum_{b=1}^{B}
g_{N_{\mathrm{pow}}}(\mu+\varepsilon^{(b)})\,\varepsilon^{(b)}.
\label{eq:mc_natgrad}
\end{equation}
The number of samples $B$ controls the variance of the estimator, while the covariance operator $\Sigma$ determines the smoothness of sampled perturbations.

We perform an explicit natural-gradient ascent step
\begin{equation}
\mu_{k+1}
=
\mu_k
+
\eta_k\,
\widehat{\nabla}_{\mathcal H_\Sigma}F_{N_{\mathrm{pow}}}(\mu_k),
\label{eq:update}
\end{equation}
where $\eta_k>0$ is a step size.
Because the update direction lies in $\mathcal H_\Sigma$, trajectory smoothness is preserved intrinsically by construction.

\subsection{Discretization for Implementation}

For numerical implementation, we discretize trajectories on a time grid $\{t_i\}_{i=1}^m\subset[0,T]$ and represent $\mu$ by its values on the grid.
We implement the covariance operator $\Sigma$ via a reproducing kernel Hilbert space (RKHS) kernel $k:[0,T]\times[0,T]\to\mathbb R$ and its kernel matrix
$K\in\mathbb R^{m\times m}$ defined by $K_{ij}=k(t_i,t_j)$.
We use a regularized kernel matrix $K_\lambda=K+\lambda I$ with $\lambda>0$ to ensure numerical stability.

Gaussian perturbations are sampled as
\begin{equation}
\varepsilon^{(s)}=\sigma K_\lambda^{1/2}z^{(s)},
\qquad z^{(s)}\sim\mathcal N(0,I_m),
\label{eq:eps_sampling}
\end{equation}
so that $\varepsilon^{(s)}\sim\mathcal N(0,\sigma^2K_\lambda)$.
A direct discretization of \eqref{eq:natgrad_closed} yields the estimator
\begin{equation}
\widehat g(\mu)
=
\frac{1}{B\sigma^2}\sum_{s=1}^{B}
\exp\!\big(N_{\mathrm{pow}} f(\mu+\varepsilon^{(s)})\big)\,
\varepsilon^{(s)}.
\label{eq:disc_estimator_raw}
\end{equation}
Then, the mean trajectory is updated according to \eqref{eq:update}.

Algorithm~\ref{alg:nfg_rkhs} summarizes the proposed optimizer in its practical RKHS discretization.
At each iteration, we draw smooth Gaussian perturbations by sampling $z^{(s)}\sim\mathcal N(0,I_m)$ and mapping them through the Cholesky factor $L$ of the regularized kernel matrix $K_\lambda$, which implements the covariance structure that controls trajectory regularity.
The update direction is estimated via \eqref{eq:disc_estimator_raw}: each perturbation is weighted by $\exp(N_{\mathrm{pow}} f(\mu+\varepsilon^{(s)}))$ and multiplied by the corresponding direction $Lz^{(s)}$, avoiding explicit matrix inversion.
The resulting Monte-Carlo estimate $\widehat g$ is then used for an explicit ascent step $\mu\leftarrow \mu+\eta_k\widehat g$, yielding smooth trajectory updates whose structure is determined by the kernel while remaining compatible with black-box, potentially nondifferentiable objectives.

\begin{remark}[Resolution-agnostic discretization]
Although Algorithm~\ref{alg:nfg_rkhs} is implemented on a finite time grid, the optimization problem is formulated in the continuous Hilbert space $\mathcal H$.
The discretization serves only as a numerical representation and can therefore be refined or changed without altering the underlying trajectory optimization problem.
\end{remark}

\begin{remark}[Finite-dimensional approximation and covariance truncation]
Although the proposed optimization method is formulated intrinsically in
the infinite-dimensional Hilbert space $\mathcal H$, the practical
implementation operates on a finite-dimensional approximation induced by
trajectory discretization and RKHS kernel representations.

Let $(e_k)_{k\ge1}$ denote the eigenbasis of the covariance operator
$\Sigma$, satisfying
$\Sigma e_k=\lambda_k e_k,$ for $\lambda_k>0$.
For a truncation level $d_{\mathrm{eff}}\ge1$, define the finite-dimensional subspace
$\mathcal H_{d_{\mathrm{eff}}}:=\operatorname{span}\{e_1,\dots,e_{d_{\mathrm{eff}}}\}\subset \mathcal H,$
and let $P_{d_{\mathrm{eff}}}:\mathcal H\to\mathcal H_{d_{\mathrm{eff}}}$ denote the $\mathcal H$-orthogonal projection $P_{d_{\mathrm{eff}}} h=\sum_{k=1}^{d_{\mathrm{eff}}}\langle h,e_k\rangle_{\mathcal H} e_k$.
Here, $d_{\mathrm{eff}}$ is interpreted as an effective dimension of
the trajectory optimization problem and is closely related to the
trajectory discretization level $m$ used in the RKHS implementation.
The projection operator introduced later plays an analogous role in the
stochastic analysis and iteration-complexity guarantees.
\end{remark}

\section{Iteration Complexity of Natural Functional Gradient}
This subsection establishes theoretical guarantees for Gaussian-smoothed natural functional gradient ascent. In particular, we show that exponential smoothing induces an optimization landscape oriented toward global optimality that eliminates spurious stationary points, and we derive a global iteration-complexity bound based on Monte–Carlo estimation.
$\xi^*$ denotes a unique optimizer of the original objective $f(\xi)$. We optimize a Gaussian-smoothed surrogate that serves as a tractable proxy in the zeroth-order setting.
\begin{theorem}[Directional ascent toward optimal point]
\label{thm:ascent}
Assume that \(g_N(x):=e^{Nf(x)}\mathbf 1_{\{x\in S\}}\) is bounded for each \(N\).
For \(M>0\), define the ellipsoid \(K_M:=\overline{\left\{\mu\in\mathcal H_{\Sigma}:\|\mu\|_{\mathcal H_{\Sigma}}\le M\right\}}.\)
Then, for every \(\epsilon>0\) and \(M>0\), there exists
\(N_0=N_0(\epsilon,M,\sigma,\Sigma)<\infty\) such that for all \(N\ge N_0\), for every \(\mu\in K_M\) satisfying $\|\mu-\xi^\star\|_{\mathcal H}\ge\epsilon$,
letting
\begin{align}
u:=(\xi^\star-\mu)/\|\xi^\star-\mu\|_{\mathcal H},   \end{align}
we have $D_{\Sigma u}F_N(\mu)>0$.
\end{theorem}

The proof can be found in the supplementary materials. Theorem 1 shows that, under stated conditions, the natural functional gradient of the surrogate is directionally aligned with $\xi^*$: outside an $\epsilon$-neighborhood of $\xi^*$, the update provides a strict ascent direction toward $\xi^*$ for sufficiently large N. Consequently, the optimization is guided toward an $\epsilon$ neighborhood of the true optimum.

\begin{remark}[Eigen-directional ascent outside an $\epsilon$-ball]
\label{rem:eigen-ascent}
Let $e_i$ be an eigenfunction of the covariance operator $\Sigma$, i.e., $\Sigma e_i=\lambda_i e_i$ with $\lambda_i>0$.
Consider an iterate of the form \(\mu=\xi^\star - t\epsilon e_i\) with \(t>1,\) so that $\|\mu-\xi^\star\|_{\mathcal H}>\epsilon$.
Then, by Theorem~\ref{thm:ascent}, the Cameron--Martin direction $\Sigma e_i$ is an ascent direction of the smoothed objective $F_N$ for all sufficiently large $N$.
Equivalently, since directional derivatives are linear and $\lambda_i>0$, the eigenfunction direction $e_i$ itself satisfies \(D_{e_i}F_N(\mu)>0.\)
Hence, outside the \(\epsilon\)-ball around \(\xi^\star\), the natural functional gradient always admits an ascent component toward the optimizer along the corresponding eigenmode. 
In particular, no point outside this neighborhood can be stationary for the smoothed objective.
\end{remark}

We now derive an iteration complexity bound for the proposed stochastic natural functional gradient method by adapting standard stochastic gradient analysis to the Cameron--Martin geometry induced by Gaussian smoothing.

\begin{lemma}[Second directional derivative identity]
\label{lem:second-directional}
Assume that 
\(g_{N_{\mathrm{pow}}}\) is bounded. Then, for any
\(h,k\in\mathcal H_\Sigma\), the mixed second directional derivative exists and satisfies
{\footnotesize\begin{align}
D^2&F_{N_{\mathrm{pow}}}(\mu)[h,k]
=\frac{1}{\sigma^{2}}\mathbb E\!\left[
g_{N_{\mathrm{pow}}}(\mu+\varepsilon)
\left(
\widehat h(\varepsilon)
\widehat k(\varepsilon)
-
\langle h,k\rangle_{\mathcal H_{\Sigma}}
\right)
\right],
\label{eq:second-directional-cm}
\end{align}}%
\allowdisplaybreaks
where $\widehat h(\varepsilon):=\sum_{j:\lambda_j>0}
\langle h,e_j\rangle_{\mathcal H}
\xi_j /\sqrt{\lambda_j}$ and $\widehat k(\varepsilon):=\sum_{j:\lambda_j>0}\langle k,e_j\rangle_{\mathcal H} \xi_j/ \sqrt{\lambda_j}$.
\(D^2F_{N_{\mathrm{pow}}}(\mu)[h,k]=D^2F_{N_{\mathrm{pow}}}(\mu)[k,h]\) also hold.
\end{lemma}
The proof can be found in the supplementary materials.
Lemma~\ref{lem:second-directional} provides a closed-form second-order identity via Gaussian integration by parts.
This identity directly yields a uniform bound on the operator norm of the Hessian along Cameron--Martin directions, leading to a Lipschitz constant that depends only on $\|g_{N_{\mathrm{pow}}}\|_\infty$ and $\sigma$.
\begin{lemma}[Lipschitz continuity of the natural functional gradient]
\label{lem:lipschitz-from-second-derivative}
Assume that \(g_{N_{\mathrm{pow}}}\) is bounded, and define $\|g_{N_{\mathrm{pow}}}\|_\infty:=\sup_{\xi\in\mathcal H}|g_{N_{\mathrm{pow}}}(\xi)|<\infty$
Then the natural functional gradient
$\nabla_{\mathcal H_{\Sigma}}F_{N_{\mathrm{pow}}}$ is Lipschitz continuous along Cameron--Martin directions: for any $\mu,\nu\in\mathcal H$ such that $\mu-\nu\in\mathcal H_{\Sigma}$, we have
{\footnotesize\begin{align}
\left\|
\nabla_{\mathcal H_\Sigma}
F_{N_{\mathrm{pow}}}(\mu)
-
\nabla_{\mathcal H_\Sigma}
F_{N_{\mathrm{pow}}}(\nu)
\right\|_{\mathcal H_\Sigma}
\le
\frac{2\|g_{N_{\mathrm{pow}}}\|_\infty}{\sigma^2}
\,
\|\mu-\nu\|_{\mathcal H_\Sigma}.
\end{align}}
\end{lemma}
The proof can be found in the supplementary materials.
We then control the stochasticity introduced by Monte-Carlo estimation.
The following lemma shows unbiasedness and a second-moment bound that scales with $d_{\mathrm{eff}}$, capturing the intrinsic effective dimension of the perturbations.
\begin{lemma}[Second moment of the projected Monte-Carlo natural gradient estimator]
\label{lem:mc-second-moment}
Let \((e_k)_{k\ge1}\) be an eigenbasis of \(\Sigma\), and define
\(\mathcal H_{d_{\mathrm{eff}}}:=\operatorname{span}\{e_1,\dots,e_{d_{\mathrm{eff}}}\}.\)
Let \(P_{d_{\mathrm{eff}}}:\mathcal H\to\mathcal H_{d_{\mathrm{eff}}}\) be the \(\mathcal H\)-orthogonal projection. Let $\widehat g(\mu)$ be the discretized Monte-Carlo estimator defined in \eqref{eq:disc_estimator_raw},
where $\varepsilon^{(b)} \overset{i.i.d.}{\sim}\mathcal N(0,\sigma^2K_{\lambda})$.
Then \(\widehat g(\mu)\) is an unbiased estimator of $P_{d_{\mathrm{eff}}}\nabla_{\mathcal H_\Sigma}F_{N_{\mathrm{pow}}}(\mu)$, and satisfies
\begin{align}
\mathbb E\!\left[
\big\|
\widehat g(\mu)
\big\|_{\mathcal H_\Sigma}^2
\right]
\le
\frac{
\|g_{N_{\mathrm{pow}}}\|_\infty^2 d_{\mathrm{eff}}
}{
B\sigma^2
}.
\end{align}
\end{lemma}
The proof can be found in the supplementary materials.
Combining the Lipschitz property with the second-moment bound yields a one-step expected ascent inequality.
\begin{lemma}[One-step expected ascent]
\label{lem:onestep-stochastic}
Let \(P_{d_{\mathrm{eff}}}:\mathcal H\to\mathcal H_{d_{\mathrm{eff}}}\) be the \(\mathcal H\)-orthogonal projection.
Suppose that \(\{\mu_t\}\) is generated by the projected update rule. Then, conditioning on \(\mu_t\),
{\small\begin{align}
\mathbb E\!\left[
F_{N_{\mathrm{pow}}}(\mu_{t+1})
\,\big|\,
\mu_t
\right]
\ge\;
&
F_{N_{\mathrm{pow}}}(\mu_t)
+
\eta_t
\big\|
P_{d_{\mathrm{eff}}}\nabla_{\mathcal H_\Sigma}
F_{N_{\mathrm{pow}}}(\mu_t)
\big\|_{\mathcal H_\Sigma}^{2}\nonumber\\
&-
\frac{
\|g_{N_{\mathrm{pow}}}\|_\infty^3 m
}{
B\sigma^4
}
\eta_t^2 .
\end{align}}
\end{lemma}%
The proof can be found in the supplementary materials.
Finally, we conclude with an iteration-complexity bound for the covariance-truncated stochastic dynamics.
The result gives a zeroth-order convergence rate, expressed in terms of $\|g_{N_{\mathrm{pow}}}\|_\infty$, $\sigma$, and $d_{\mathrm{eff}}$.

\begin{theorem}[Iteration complexity in a Hilbert space]
\label{thm:complexity}
Let \(\{\mu_t\}_{t\ge 0}\subset\mathcal H\) be generated by the projected update rule with a constant stepsize \(\eta>0\). Then the iterates satisfy
\begin{align}
\min_{0\le t\le T-1}
\mathbb E\!\left[
\big\|
P_{d_{\mathrm{eff}}}\nabla_{\mathcal H_\Sigma}
F_{N_{\mathrm{pow}}}(\mu_t)
\big\|_{\mathcal H_\Sigma}
\right]
\le
\varepsilon
\end{align}
provided that
$T
\ge
\frac{
4\,
\|g_{N_{\mathrm{pow}}}\|_\infty^{4}\,
d_{\mathrm{eff}}
}{
B\,\sigma^{4}\varepsilon^{4}
}$.
\end{theorem}
The proof is provided in the supplementary materials.
Here, \(d_{\mathrm{eff}}\) denotes the effective dimension induced by the truncated covariance eigenbasis.
The iteration complexity scales linearly with \(d_{\mathrm{eff}}\), yielding an overall complexity of \(\mathcal O(d_{\mathrm{eff}}\varepsilon^{-4})\) for reaching \(\varepsilon\)-stationarity.
To our knowledge, this is the first global convergence-rate result for trajectory optimization formulated directly in function space.

\begin{remark}[Comparison with Existing Trajectory Optimization Methods]
The proposed method is related to sampling-based planners such as RRT*, which are asymptotically optimal but often require additional smoothing for discrete paths \cite{karaman2011rrtstar}. In contrast, our method directly optimizes smooth trajectories through the Cameron--Martin geometry.

Unlike CHOMP and TrajOpt, which rely on analytic gradients and discretization-dependent regularization \cite{zucker2013chomp,schulman2014trajopt}, the proposed approach does not require differentiable costs and remains applicable under nondifferentiable collision checking. 

From a zeroth-order optimization viewpoint, the resulting
$\mathcal O(\varepsilon^{-4})$ complexity matches standard Gaussian smoothing rates while explicitly revealing how the covariance operator controls exploration and trajectory regularity \cite{nesterov2017random,gao2022generalizing}.
\end{remark}

\section{Experiments}\label{sec:experiments}
This section evaluates the proposed natural functional gradient (NFG) trajectory optimizer across progressively more constrained settings. 
We first study optimization behavior in a synthetic narrow-passage environment (\S\ref{susbsec:synthetic}), where the feasible region is fragmented and highly sensitive to trajectory perturbations. 
We then evaluate manipulation performance in cluttered cabinet insertion tasks (\S\ref{susbsec:simulation}), including unified execution validation under shared simulation conditions. 
Finally, we demonstrate real-world execution on a physical robotic platform (\S\ref{susbsec:real}) to illustrate the practical feasibility and smooth execution behavior of the proposed function-space trajectory optimization framework.

\subsection{Synthetic Narrow-Passage Experiment}\label{susbsec:synthetic}
We first evaluate the proposed method in a synthetic narrow-passage trajectory optimization problem designed to study optimization behavior under fragmented feasible regions. 
The environment consists of multiple box-shaped obstacles arranged along the temporal axis, producing a narrow collision-free corridor that occupies only a small subset of the trajectory space. 
Because small trajectory perturbations can easily induce collision, the environment provides a controlled benchmark for evaluating optimization behavior under constrained geometry.

The task is formulated as a one-dimensional trajectory optimization problem over a fixed $1\,\mathrm{s}$ horizon discretized at $100\,\mathrm{Hz}$. 
The objective combines collision avoidance with a trajectory regularity bonus based on the average absolute jerk of the trajectory. 
Additional details and environment specifications are provided in the supplementary material.

For the proposed method, trajectory perturbations are generated using a squared exponential (SE) kernel, defined as $k_{\mathrm{SE}}(x,x') = g^{2}\exp\big(-\tfrac{1}{2l^{2}}|x-x'|^{2}\big)$, with variance $g^{2}=0.29$ and length-scale $l=0.22$. 
The exponential smoothing parameter is set to $N_{\mathrm{pow}}=100$, and each optimization step uses $100$ Monte-Carlo trajectory perturbations. 
We compare against representative trajectory optimization baselines including CHOMP~\cite{zucker2013chomp}, STOMP~\cite{kalakrishnan2011stomp}, and MPPI~\cite{williams2015model}. 
All methods are initialized using the same linear interpolation trajectory between the initial and final states and evaluated under the same trajectory horizon and temporal discretization.
NFG, CHOMP \cite{zucker2013chomp}, and STOMP \cite{kalakrishnan2011stomp} are executed for $100$ iterations, while MPPI~\cite{williams2015model} is evaluated using its standard online planning configuration. 
All experiments are repeated over five random seeds while using the same obstacle configuration and optimization setup across seeds.
Performance is evaluated using success rate, runtime, path length, and average jerk. 
A trajectory is counted as successful if it remains collision-free throughout the trajectory horizon. 
Average jerk is computed from successful trajectories only.

Table~\ref{tab:synthetic} summarizes the quantitative results. 
The proposed method achieves the highest success rate among the baselines, consistently discovering feasible trajectories despite the fragmented feasible region. 
In contrast, local trajectory optimization methods, e.g., CHOMP~\cite{zucker2013chomp} and STOMP~\cite{kalakrishnan2011stomp}, frequently fail to enter the narrow feasible corridor, while MPPI~\cite{williams2015model} occasionally discovers feasible solutions but produces trajectories with substantially larger jerk. 
Although some baselines achieve lower jerk values in individual successful trials, the proposed method maintains a more consistent balance between feasibility and trajectory regularity across repeated runs.

\subsection{Simulation Experiments in Cabinet Environments}\label{susbsec:simulation}
We next evaluate the proposed method on robotic manipulation tasks in cluttered cabinet environments with progressively constrained geometry. 
The experiments study trajectory generation under narrow collision-free clearances, where feasible motions require stable trajectory deformation near the insertion region. 
We first describe the cabinet environments and optimization setup (\S\ref{subsubsec:setup}), followed by the execution-level evaluation protocol (\S\ref{subsubsec:evaluation}) and quantitative comparisons against representative planning and trajectory optimization baselines (\S\ref{subsubsec:results}).

\begin{table}[t]
\centering
\footnotesize
\caption{Comparison with the baseline planners in the synthetic box-avoidance scenario.}
\label{tab:synthetic}
\setlength{\tabcolsep}{3.5pt}
\renewcommand{\arraystretch}{1.05}
\resizebox{\columnwidth}{!}{
\begin{tabular}{l S S S S}
\toprule
\textbf{Method} &
\textbf{Success Rate $\uparrow$} &
\textbf{Time $\downarrow$} &
\textbf{Path Length $\downarrow$} &
\textbf{Avg Jerk $\downarrow$} \\
&
\textbf{(\%)} &
\textbf{($\mathrm{s}$)} &
\textbf{($\mathrm{m})$} &
\textbf{($\mathrm{m/s^3}$)} \\
\midrule
\textbf{NFG (Ours)}
& \textbf{{100.0}}
& {0.73 $\pm$ 0.03}
& \textbf{{10.06 $\pm$ 1.25}}
& {90.91 $\pm$ 33.43} \\
STOMP \cite{kalakrishnan2011stomp}
& {20.0}
& {0.63 $\pm$ 0.01}
& {10.62 $\pm$ 0.00}
& \textbf{{34.67 $\pm$ 0.00}} \\
CHOMP \cite{zucker2013chomp}
& {20.0}
& \textbf{{0.02 $\pm$ 0.00}}
& {10.82 $\pm$ 0.00}
& {43.52 $\pm$ 0.00} \\
MPPI \cite{williams2015model}
& {60.0}
& {0.37 $\pm$ 0.00}
& {7.68 $\pm$ 0.43}
& {621.22 $\pm$ 86.07} \\
\bottomrule
\end{tabular}
}
\vspace{-2mm}
\end{table}

\subsubsection{Environment Setup}\label{subsubsec:setup}
We evaluate the proposed method using a Franka Research 3 manipulator in a cabinet insertion task. 
The robot inserts a cylindrical object into a cabinet while avoiding collisions with surrounding obstacles and cabinet walls. 
Because the insertion region contains limited geometric clearance, small trajectory deviations near the cabinet opening can induce collision during execution. 
To study performance under progressively constrained geometry, we consider four environments with increasing levels of difficulty: free space, fully open, quarter-closed, and half-closed, as illustrated in Fig.~\ref{fig:cabinet}. 
The free space environment is included as a sanity-check setting without cabinet constraints, while the remaining environments introduce progressively narrower cabinet openings around the insertion region. 
All environments share the same initial and final joint configurations, allowing the effect of geometric constraints on trajectory generation to be evaluated consistently across environments.

The trajectory objective combines collision avoidance with joint-space motion regularization. 
Collision costs are computed from penetration depth and contact information obtained during simulation, while an additional path-length regularization term penalizes excessive joint-space motion along the trajectory horizon. 
This formulation encourages trajectories that avoid collision while limiting unnecessarily aggressive configuration changes during execution. 
All methods are initialized using the same joint-space linear interpolation trajectory between the start and goal configurations and evaluated under the same simulation and execution settings.

\begin{figure}[t]
    \centering
    \includegraphics[width=0.99\linewidth]{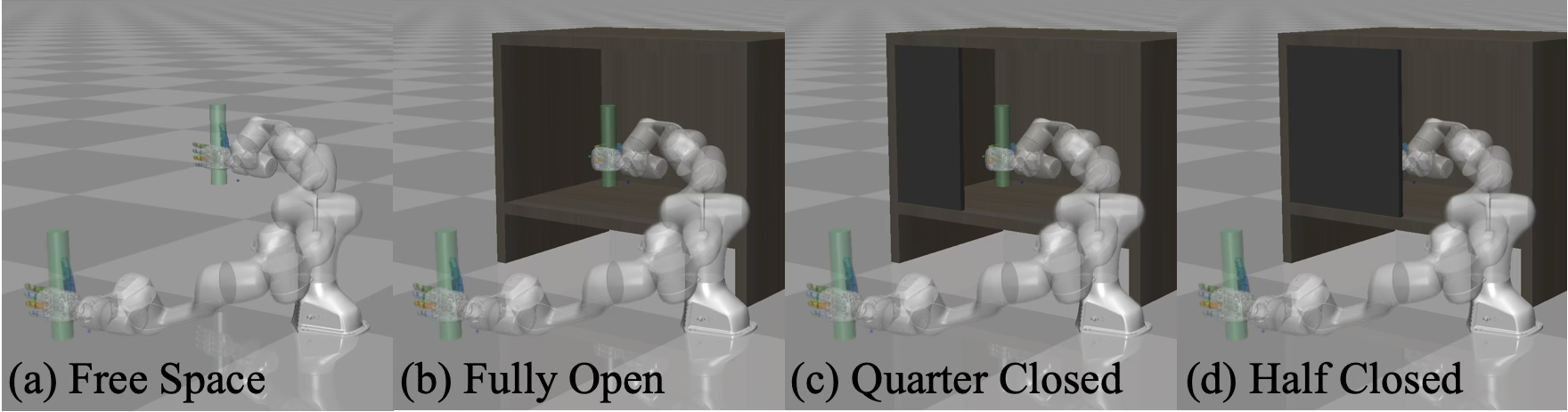}
    \vspace{-3mm}
    \caption{Cabinet insertion environments with increasing geometric constraints: (a) free space, (b) fully open, (c) quarter-closed, and (d) half-closed.}
    \label{fig:cabinet}
    \vspace{-10pt}
\end{figure}

\begin{figure*}[t]
    \centering
    {\includegraphics[width=\linewidth]{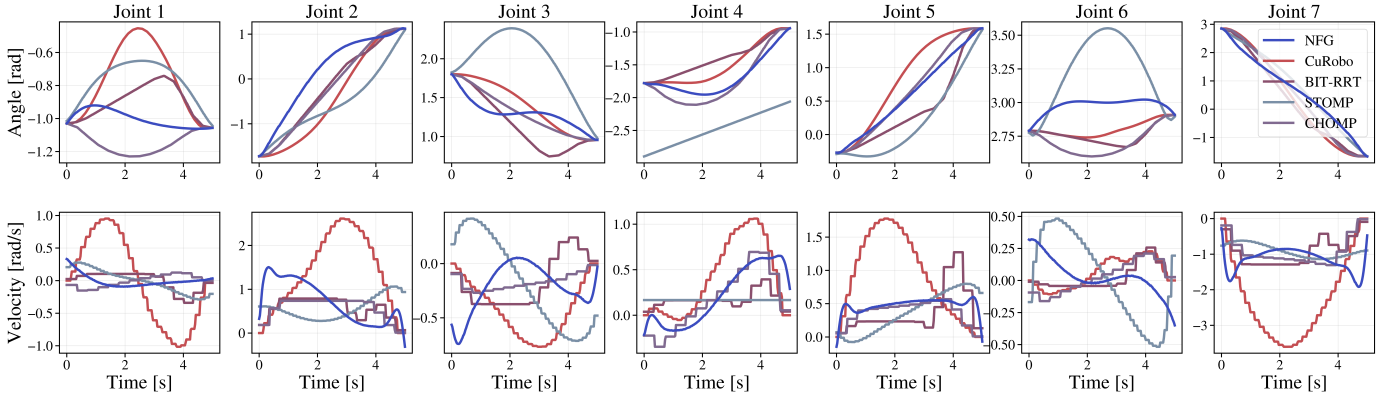}}
    \caption{
    Joint trajectories for the fully open cabinet environment. 
    The top and bottom rows show the joint positions $\mathrm{(rad)}$ and joint velocities $\mathrm{(rad/s)}$ of the Franka Research 3 manipulator, respectively. 
    STOMP~\cite{kalakrishnan2011stomp} fails to generate a feasible trajectory.
}
 \vspace{-10pt}
 \label{fig:sim_joint_plot}
\end{figure*}

Following the synthetic narrow-passage experiment, the proposed method uses a squared exponential (SE) kernel to generate smooth trajectory perturbations in function space. 
The kernel variance is set to $g^{2}=1.0$, and the length-scale is chosen as half of the trajectory horizon. 
Trajectory optimization is performed over a fixed $5\,\mathrm{s}$ horizon discretized at $100\,\mathrm{Hz}$ using $100$ sampled trajectories per optimization step with $N_{\mathrm{pow}}=20$. 
Each optimization run is executed for at most $30$ iterations or $300\,\mathrm{s}$ and terminates early once a collision-free trajectory is found.

We compare the proposed method against representative motion planning and trajectory optimization baselines, including BiTRRT~\cite{jaillet2010sampling}, CHOMP~\cite{zucker2013chomp}, STOMP~\cite{kalakrishnan2011stomp}, and cuRobo~\cite{sundaralingam2023curobo}. 
These methods cover sampling-based planning, gradient-based trajectory optimization, stochastic trajectory optimization, and GPU-accelerated motion generation, respectively. 
All baselines are evaluated using the same start and goal configurations, trajectory discretization, and execution settings. 
For each cabinet environment, all methods are evaluated over five random seeds under the same environment geometry and task configuration. 
Additional implementation details and evaluation settings are provided in the supplementary material.

\subsubsection{Unified Execution Evaluation}\label{subsubsec:evaluation}
The compared planners employ different internal trajectory representations, optimization strategies, and collision models. 
Some methods generate sparse geometric waypoints~\cite{jaillet2010sampling}, while others directly optimize discretized joint trajectories~\cite{zucker2013chomp,kalakrishnan2011stomp} or rely on planner-specific collision approximations during optimization~\cite{sundaralingam2023curobo}.
As a result, the native outputs of different planners are not directly comparable under a shared execution setting without additional trajectory processing and validation.

To enable a consistent comparison, all generated trajectories are converted into a unified time-parameterized joint-space trajectory and resampled at a fixed execution frequency of $100\,\mathrm{Hz}$. 
The resampled trajectories are then evaluated in the MuJoCo simulator using the same collision-checking and execution settings across all methods. 
A rollout is considered successful if the resampled execution trajectory remains collision-free throughout the full trajectory horizon.

In addition to success rate, we report total optimization time and average jerk computed from the executed trajectories. 
Average jerk is reported only for successful trajectories and is used to evaluate trajectory regularity during execution.

\begin{table}[t]
\centering
\scriptsize
\renewcommand{\arraystretch}{1.25}
\caption{Comparison with baseline planners across cabinet environments.}
\label{tab:cabinet}
\vspace{-10pt}
\begin{subtable}{\columnwidth}
\centering
\setlength{\tabcolsep}{7pt}
\caption{\textbf{Success Rate ($\uparrow$, $\mathrm{\%}$)}}
\label{tab:cabinet_success}
\vspace{-2pt}
\begin{tabularx}{\columnwidth}{@{} l S S S S @{}}
\toprule
\textbf{Method} & {\textbf{Free Space}} & {\textbf{Fully Open}} & {\textbf{Quarter-Closed}} & {\textbf{Half-Closed}} \\
\midrule
\textbf{NFG (Ours)}                   & \textbf{{100.0}} & \textbf{{100.0}} & \textbf{{100.0}} & \textbf{{100.0}} \\
BiTRRT \cite{jaillet2010sampling}     & \textbf{{100.0}} & {40.0}   & {13.3}    & {16.7}   \\
STOMP  \cite{kalakrishnan2011stomp}   & \textbf{{100.0}} & {0.0}    & {0.0}    & {0.0}    \\
CHOMP  \cite{zucker2013chomp}         & \textbf{{100.0}} & \textbf{{100.0}} & {0.0}    & {0.0}    \\
cuRobo \cite{sundaralingam2023curobo} & \textbf{{100.0}} & \textbf{{100.0}} & \textbf{{100.0}} & \textbf{{100.0}}  \\
\bottomrule
\end{tabularx}
\end{subtable}

\vspace{5pt}

\begin{subtable}{\columnwidth}
\centering
\setlength{\tabcolsep}{5.0pt}
\caption{\textbf{Optimization Time ($\downarrow$, $\mathrm{s}$)}}
\vspace{-2pt}
\label{tab:cabinet_time}
\begin{tabularx}{\columnwidth}{@{} l S S S S @{}}
\toprule
\textbf{NFG (Ours)}                   & {0.31 $\pm$ 0.01}  & {4.10 $\pm$ 2.89}      & {103.68 $\pm$ 75.29}  & {166.10 $\pm$ 91.63} \\
BiTRRT \cite{jaillet2010sampling}     & {0.19 $\pm$ 0.02}  & \textbf{{0.30 $\pm$ 0.04}}      & \textbf{{0.51 $\pm$ 0.29}}  & \textbf{{0.82 $\pm$ 0.18}}  \\
STOMP  \cite{kalakrishnan2011stomp}   & \textbf{{0.00 $\pm$ 0.00}}  & {-}      & {-}  & {-}  \\
CHOMP  \cite{zucker2013chomp}         & {0.13 $\pm$ 0.01}  & {0.31 $\pm$ 0.10}      & {-}  & {-}  \\
cuRobo \cite{sundaralingam2023curobo} & {0.01 $\pm$ 0.00}  & {0.020 $\pm$ 0.01}  & {14.24 $\pm$ 0.17}  & {128.87 $\pm$ 0.18}  \\
\bottomrule
\end{tabularx}
\end{subtable}

\vspace{5pt}

\begin{subtable}{\columnwidth}
\centering
\setlength{\tabcolsep}{2.1pt}
\caption{\textbf{Average Jerk ($\downarrow$, $\mathrm{rad/s^3}$)}}
\label{tab:cabinet_jerk}
\vspace{-2pt}
\begin{tabularx}{\columnwidth}{@{} l S S S S @{}}
\toprule
\textbf{NFG (Ours)}                   & {461.61 $\pm$ 178.53}  & {461.61 $\pm$ 178.52}      & \textbf{{694.06 $\pm$ 275.11}}  & \textbf{{840.40 $\pm$ 370.27}} \\
STOMP  \cite{kalakrishnan2011stomp}   & \textbf{{0.82 $\pm$ 0.18}}  & {-}      & {-}  & {-}  \\
CHOMP  \cite{zucker2013chomp}         & {57.62 $\pm$ 0.05}  & \textbf{{109.21 $\pm$ 6.39}}      & {-}  & {-}  \\
cuRobo \cite{sundaralingam2023curobo} & {1141.31 $\pm$ 39.28}  & {859.73 $\pm$ 11.19}      & {1219.58 $\pm$ 54.75}   & {1177.27 $\pm$ 79.62}  \\
\bottomrule
\end{tabularx}
\end{subtable}

\vspace{-15pt}
\end{table}

\subsubsection{Quantitative Results}\label{subsubsec:results}
Table~\ref{tab:cabinet} summarizes the quantitative results across the cabinet environments. 
We report success rate, total optimization time, and average jerk measured from the executed trajectories after trajectory resampling and simulation-based validation. 
A trajectory is counted as successful if the resampled execution trajectory remains collision-free throughout the full trajectory horizon. 
Average jerk is reported only for successful trajectories and is used to evaluate trajectory regularity during execution.

In the free space and fully open cabinet environments, multiple planners, including the proposed method, BiTRRT~\cite{jaillet2010sampling}, and cuRobo~\cite{sundaralingam2023curobo}, are able to generate collision-free trajectories under the shared execution setting. 
Because the geometric constraints are relatively mild in these scenarios, feasible motions can often be discovered without substantial trajectory deformation. 
As the cabinet opening becomes smaller, however, the feasible region becomes increasingly narrow, making feasible trajectory generation significantly more difficult in the quarter-closed and half-closed environments.

Under these more constrained settings, the proposed method maintains successful trajectory generation across all cabinet configurations while producing lower average jerk than cuRobo~\cite{sundaralingam2023curobo}. 
The joint trajectories in Fig.~\ref{fig:sim_joint_plot} exhibit gradual configuration changes throughout the insertion motion without abrupt transitions near the constrained region. 
In contrast, sampling-based and optimization-based baselines show reduced success rates as the cabinet opening becomes narrower.

cuRobo~\cite{sundaralingam2023curobo} also achieves high feasibility across the evaluated cabinet environments while demonstrating substantially faster optimization times. 
However, the resulting trajectories exhibit larger execution jerk under the shared evaluation setting. 
Overall, the results suggest that the proposed method provides stable trajectory generation in cluttered manipulation environments while maintaining smooth trajectory evolution near narrow geometric constraints.

\begin{figure*}[t]
    \centering
    {\includegraphics[width=\linewidth]{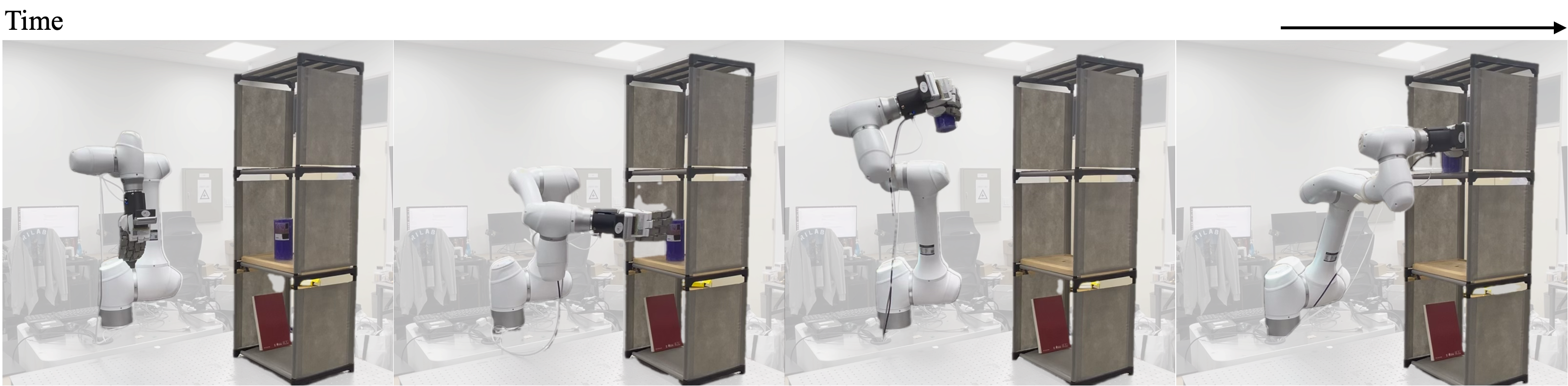}}
    \caption{Real-world robotic manipulation in a cluttered environment with a narrow box constraint.}
 \vspace{-10pt}
 \label{fig:real}
\end{figure*}

\subsection{Real-World Robotic Demonstration}\label{susbsec:real}
We further evaluate the proposed method on a real-world manipulation task using a Doosan Robotics A0912 manipulator in a cluttered bookshelf environment. 
The task requires the robot to transfer a cylindrical object from a second-level shelf to a target placement region located inside a narrow compartment on the third level of the bookshelf. 
Because the target region is partially enclosed by surrounding shelf structures, feasible motions require accurate trajectory generation near constrained workspace boundaries.

The initial and final robot configurations are determined from inverse kinematics solutions corresponding to the grasp pose and target placement pose of the cylindrical object, respectively. 
Trajectory optimization is initialized using a joint-space linear interpolation trajectory between the initial and final configurations. 
The optimized trajectories are generated over a fixed $5\,\mathrm{s}$ horizon discretized at $100\,\mathrm{Hz}$ and executed directly on the robot using the same temporal discretization.

Figure~\ref{fig:real} shows representative snapshots from the real-world experiment. 
The robot successfully completes the transfer and insertion task while maintaining stable motion near the constrained shelf geometry. 
These results demonstrate that the proposed method can generate feasible and smooth motions in cluttered real-world manipulation environments.

\section{Discussion and Limitations}
This work demonstrates that natural functional gradient optimization provides a practical framework for trajectory generation in constrained robotic manipulation environments. 
By operating directly in function space and applying Gaussian smoothing, the proposed method generates smooth trajectory perturbations while remaining compatible with zeroth-order optimization and black-box trajectory evaluation. 
The resulting framework allows feasibility and trajectory regularity to be handled within a unified optimization procedure, particularly in environments with narrow geometric clearances and fragmented feasible regions.

Table~\ref{tab:method_comparison} summarizes conceptual differences between the proposed method and representative planning and trajectory optimization approaches. 
Unlike methods that rely on geometric primitives or approximated collision representations, the proposed framework directly incorporates simulator-based contact evaluation during optimization. 
This allows trajectory updates to be computed using the same simulation environment employed during execution without requiring additional analytical approximations of collision geometry. 
The formulation also operates as a unified framework that jointly handles smooth trajectory generation and collision-aware optimization while remaining compatible with parallel rollout-based evaluation.

The primary limitation of the proposed approach is the additional computational cost introduced by Monte-Carlo estimation of the natural functional gradient. 
Although the optimization procedure is naturally parallelizable, the required number of sampled trajectory perturbations can still become expensive in high-dimensional problems or time-critical settings. 
Improving sample efficiency and exploiting large-scale parallel computation, similar to recent GPU-accelerated motion generation systems \cite{sundaralingam2023curobo}, remain important directions.

The performance of the method additionally depends on the covariance operator used to generate trajectory perturbations. 
Although fixed squared exponential kernels are effective across the evaluated tasks, performance can vary depending on the smoothing scale and covariance structure \cite{nesterov2017random,gao2022generalizing}. 
Designing adaptive or task-dependent covariance operators that better reflect environmental geometry, contact structure, or task constraints remains an important open problem.

\begin{table}[t]
\centering
\scriptsize
\setlength{\tabcolsep}{4pt}
\renewcommand{\arraystretch}{1.15}
\caption{Comparison of motion planning methods}
\label{tab:method_comparison}
\begin{tabular}{lccccc}
\toprule
\textbf{Method}
& \textbf{Smooth} 
& \textbf{No Primitive} 
& \textbf{Simulator} 
& \textbf{Single} 
& \textbf{Parallel} \\
& \textbf{Trajectory} 
& \textbf{Approx.} 
& \textbf{Contact Solver} 
& \textbf{Package} 
& \textbf{-izable} \\
\midrule
Bi-RRT \cite{jaillet2010sampling} 
&  & \checkmark &  &  & \checkmark \\
STOMP \cite{kalakrishnan2011stomp}
& \checkmark &  &  & \checkmark & \checkmark \\
CHOMP \cite{zucker2013chomp}
& \checkmark &  &  & \checkmark &  \\
MPPI \cite{williams2015model}
& \checkmark & \checkmark & \checkmark &  & \checkmark \\
CuRobo \cite{sundaralingam2023curobo}
& \checkmark &  &  & \checkmark & \checkmark \\
\textbf{Ours}
& \checkmark & \checkmark & \checkmark & \checkmark & \checkmark \\
\bottomrule
\end{tabular}
\vspace{-10pt}
\end{table}

\vspace{-3pt}
\section{Conclusion}
This paper presented a trajectory optimization framework that performs geometry-aware updates directly in function space using natural functional gradients induced by Gaussian smoothing. 
The proposed method enables zeroth-order trajectory optimization with smooth function-space perturbations without relying on analytic gradients or discretization-dependent smoothness penalties. 
We additionally analyzed the Gaussian-smoothed objective directly in Hilbert space and established convergence guarantees for the resulting optimization procedure.
Experiments on constrained robotic manipulation tasks demonstrated that the proposed method improves success rate and produces smoother execution trajectories than representative planning and trajectory optimization baselines in cluttered environments with narrow geometric clearances. 
These results suggest that function-space natural functional gradient optimization provides a practical framework for trajectory generation when feasible regions are sparse and reliable gradient information is difficult to obtain.

\vspace{-3pt}
\section*{Acknowledgement}
We thank RLWRLD for generously supporting our experiments by providing access to the Franka Research 3 robot platform and the Inspire Hand system. 
\vspace{-5pt}

\bibliographystyle{ieeetr}
\bibliography{references}

\clearpage
\onecolumn
\appendices
\section*{Supplementary Appendix}
\addcontentsline{toc}{section}{Supplementary Appendix}

\makeatletter
\@addtoreset{equation}{section}
\@addtoreset{figure}{section}
\@addtoreset{table}{section}
\@addtoreset{lemma}{section}
\@addtoreset{theorem}{section}
\@addtoreset{assumption}{section}
\makeatother
\renewcommand{\thesection}{\Alph{section}}
\renewcommand{\theequation}{\Alph{section}.\arabic{equation}}
\renewcommand{\thefigure}{\Alph{section}.\arabic{figure}}
\renewcommand{\thetable}{\Alph{section}.\arabic{table}}
\renewcommand{\thelemma}{\Alph{section}.\arabic{lemma}}
\renewcommand{\thetheorem}{\Alph{section}.\arabic{theorem}}
\renewcommand{\theassumption}{\Alph{section}.\arabic{assumption}}
\setcounter{section}{0}
\setcounter{equation}{0}
\setcounter{figure}{0}
\setcounter{table}{0}
\setcounter{lemma}{0}
\setcounter{theorem}{0}
\setcounter{assumption}{0}

\section{Proof}
\begin{lemma}[Directional derivative identity]
\label{app:lem:cm_directional}
Let \(\varepsilon\sim\mathcal N(0,\sigma^2\Sigma)\).
Then, for any \(\mu\in\mathcal H\) and any
\(h\in\mathcal H_\Sigma=\operatorname{Range}(\Sigma^{1/2})\),
the directional derivative of \(F_{N_{\mathrm{pow}}}\) exists and satisfies
\begin{equation}
D_hF_{N_{\mathrm{pow}}}(\mu)
=
\mathbb E\!\left[
g_{N_{\mathrm{pow}}}(\mu+\varepsilon)
\sum_{k:\lambda_k>0}
\frac{
\langle h,e_k\rangle_{\mathcal H}
}{
\sigma\sqrt{\lambda_k}
}
\xi_k
\right],
\label{app:eq:Dh_identity}
\end{equation}
where
\begin{align}
\varepsilon
=
\sigma
\sum_{k:\lambda_k>0}
\sqrt{\lambda_k}\,\xi_k e_k,
\qquad
\xi_k\overset{i.i.d.}{\sim}\mathcal N(0,1),
\end{align}
is the Karhunen--Lo\`eve representation associated with the eigendecomposition \(\Sigma e_k=\lambda_k e_k\).
\end{lemma}

\begin{proof}
Let
\begin{align}
\gamma=\mathcal N(0,\sigma^2\Sigma).
\end{align}
By definition,
\begin{align}
F_{N_{\mathrm{pow}}}(\mu)
=
\int_{\mathcal H}
g_{N_{\mathrm{pow}}}(\mu+x)\,\gamma(dx).
\end{align}
For \(t\in\mathbb R\),
\begin{align}
F_{N_{\mathrm{pow}}}(\mu+th)
=
\int_{\mathcal H}
g_{N_{\mathrm{pow}}}(\mu+th+x)\,\gamma(dx).
\end{align}
Introducing the change of variables \(y=x+th\), we obtain
\begin{align}
F_{N_{\mathrm{pow}}}(\mu+th)
=
\int_{\mathcal H}
g_{N_{\mathrm{pow}}}(\mu+y)\,\gamma_t(dy),
\end{align}
where
\begin{align}
\gamma_t(A)=\gamma(A-th).
\end{align}

Since \(h\in\mathcal H_\Sigma=\operatorname{Range}(\Sigma^{1/2})\), we also have
\(th\in\operatorname{Range}(\Sigma^{1/2})\). Hence the Cameron--Martin theorem implies
\(\gamma_t\ll\gamma\), with Radon--Nikodym derivative
\begin{align}
\frac{d\gamma_t}{d\gamma}(y)
=
\exp\!\left(
\frac{t}{\sigma}\widehat h(y)
-\frac{1}{2\sigma^{2}} t^2\|h\|_{\mathcal H_{\Sigma}}^2
\right),
\end{align}
where \(\widehat h\) is the Gaussian measurable linear functional associated with \(h\).
Set
\begin{align}
R_t(y)
:=
\exp\!\left(
\frac{t}{\sigma}\widehat h(y)
-\frac{1}{2\sigma^{2}} t^2\|h\|_{\mathcal H_{\Sigma}}^2
\right).
\end{align}
Then
\begin{align}
F_{N_{\mathrm{pow}}}(\mu+th)
=
\int_{\mathcal H}
g_{N_{\mathrm{pow}}}(\mu+y)R_t(y)\,\gamma(dy).
\end{align}
Therefore,
\begin{align}
\frac{
F_{N_{\mathrm{pow}}}(\mu+th)-F_{N_{\mathrm{pow}}}(\mu)
}{t}
=
\mathbb E\!\left[
g_{N_{\mathrm{pow}}}(\mu+\varepsilon)
\frac{R_t(\varepsilon)-1}{t}
\right].
\end{align}

Now write
\begin{align}
\varepsilon
=
\sigma
\sum_{k:\lambda_k>0}
\sqrt{\lambda_k}\,\xi_k e_k,
\qquad
\xi_k\overset{i.i.d.}{\sim}\mathcal N(0,1).
\end{align}
For \(h\in\mathcal H_\Sigma\), define \(h_k=\langle h,e_k\rangle_{\mathcal H}\). The Gaussian
linear functional \(\widehat h\) is given by
\begin{align}
\widehat h(\varepsilon)
=
\sum_{k:\lambda_k>0}
\frac{h_k}{\sqrt{\lambda_k}}\xi_k.
\end{align}
This series converges in \(L^2\), because
\begin{align}
\mathbb E\!\left[
\left|
\sum_{k:\lambda_k>0}
\frac{h_k}{\sqrt{\lambda_k}}\xi_k
\right|^2
\right]
=
\sum_{k:\lambda_k>0}
\frac{h_k^2}{\lambda_k}
=
\|\Sigma^{-1/2}h\|_{\mathcal H}^2
<\infty.
\end{align}

Thus,
\begin{align}
\left.\frac{d}{dt}R_t(\varepsilon)\right|_{t=0}
=
\frac{\widehat h(\varepsilon)}{\sigma}
=
\sum_{k:\lambda_k>0}
\frac{\langle h,e_k\rangle_{\mathcal H}}
{\sigma\sqrt{\lambda_k}}\xi_k.
\end{align}
Since \(g_{N_{\mathrm{pow}}}\) is bounded and \(R_t(\varepsilon)\) is an exponential of a one-dimensional Gaussian random variable, differentiation under the expectation is justified. Taking \(t\to0\) gives
\begin{align}
D_h F_{N_{\mathrm{pow}}}(\mu)
=
\mathbb E\!\left[
g_{N_{\mathrm{pow}}}(\mu+\varepsilon)
\widehat h(\varepsilon)
\right].
\end{align}
\end{proof}
\begin{lemma}[Closed form of the natural functional gradient]
\label{app:lem:natgrad_closed}
There exists a unique element
\begin{align}
\nabla_{H_{\Sigma}} F_{N_{\mathrm{pow}}}(\mu)\in\mathcal H_{\Sigma}
\end{align}
such that
\begin{align}
D_hF_{N_{\mathrm{pow}}}(\mu)
=
\left\langle
\nabla_{H_{\Sigma}} F_{N_{\mathrm{pow}}}(\mu),
h
\right\rangle_{\mathcal H_{\Sigma}}
\qquad
\text{for all }h\in\mathcal H_\Sigma,
\end{align}
and it admits the closed-form representation
\begin{equation}
\nabla_{H_{\Sigma}} F_{N_{\mathrm{pow}}}(\mu)
=
\sigma^{-2}
\mathbb E\!\left[
g_{N_{\mathrm{pow}}}(\mu+\varepsilon)\varepsilon
\right].
\label{app:eq:natgrad_closed}
\end{equation}
\end{lemma}

\begin{proof}
By Lemma~\ref{app:lem:cm_directional}, for every
\(h\in\mathcal H_{\Sigma}\),
\begin{align}
D_hF_{N_{\mathrm{pow}}}(\mu)
=
\mathbb E\!\left[
g_{N_{\mathrm{pow}}}(\mu+\varepsilon)
\sum_{k:\lambda_k>0}
\frac{\langle h,e_k\rangle_{\mathcal H}}
{\sigma\sqrt{\lambda_k}}
\xi_k
\right],
\end{align}
where
\begin{align}
\varepsilon
=
\sigma
\sum_{k:\lambda_k>0}
\sqrt{\lambda_k}\,\xi_k e_k,
\qquad
\xi_k\overset{i.i.d.}{\sim}\mathcal N(0,1).
\end{align}

Define
\begin{align}
\Lambda_\mu(h)
:=
D_hF_{N_{\mathrm{pow}}}(\mu).
\end{align}

Since \(g_{N_{\mathrm{pow}}}\) is bounded, say
\begin{align}
\|g_{N_{\mathrm{pow}}}\|_\infty\le M,
\end{align}
we obtain
\begin{align}
|\Lambda_\mu(h)|
\le
M\,
\mathbb E\!\left[
\left|
\sum_{k:\lambda_k>0}
\frac{\langle h,e_k\rangle_{\mathcal H}}
{\sigma\sqrt{\lambda_k}}
\xi_k
\right|
\right].
\end{align}

The random variable
\begin{align}
\sum_{k:\lambda_k>0}
\frac{\langle h,e_k\rangle_{\mathcal H}}
{\sigma\sqrt{\lambda_k}}
\xi_k
\end{align}
is centered Gaussian with variance
\begin{align}
\sum_{k:\lambda_k>0}
\frac{\langle h,e_k\rangle_{\mathcal H}^2}
{\sigma^2\lambda_k}
=
\frac{1}{\sigma^2}\|h\|_{\mathcal H_{\Sigma}}^2.
\end{align}
Since \(h\in\mathcal H_{\Sigma}=\operatorname{Range}(\Sigma^{1/2})\), the element
\(\Sigma^{-1/2}h\) is well-defined as the unique element in
\(\overline{\operatorname{Range}(\Sigma)}\) satisfying
\begin{align}
h=(\Sigma)^{1/2}(\Sigma)^{-1/2}h.
\end{align}
In the eigenbasis of \(\Sigma\), we have
\begin{align}
\Sigma e_k=\lambda_k e_k.
\end{align}
Hence
\begin{align}
\Sigma^{-1/2}h
=
\sum_{k:\lambda_k>0}
\frac{\langle h,e_k\rangle_{\mathcal H}}
{\sqrt{\lambda_k}}e_k,
\end{align}
and therefore
\begin{align}
\|h\|_{\mathcal H_{\Sigma}}^2
=
\|\Sigma^{-1/2}h\|_{\mathcal H}^2
=
\sum_{k:\lambda_k>0}
\frac{\langle h,e_k\rangle_{\mathcal H}^2}
{\lambda_k}
<\infty.
\end{align}
Thus the series is well-defined precisely because
\(h\in\mathcal H_{\Sigma}\).

Hence
\begin{align}
|\Lambda_\mu(h)|
\le
M\sqrt{\frac{2}{\pi}}
\,
\frac{\|h\|_{\mathcal H_{\Sigma}}}{\sigma^2}.
\end{align}

Therefore \(\Lambda_\mu\) is a bounded linear functional on the Hilbert space
\(\mathcal H_{\Sigma}\). By the Riesz representation theorem, there exists a unique
\begin{align}
G(\mu)\in\mathcal H_{\Sigma}
\end{align}
such that
\begin{align}
\Lambda_\mu(h)
=
\langle G(\mu),h\rangle_{\mathcal H_{\Sigma}}
\qquad
\forall h\in\mathcal H_{\Sigma}.
\end{align}

We now identify \(G(\mu)\). Define
\begin{align}
m(\mu)
:=
\mathbb E\!\left[
g_{N_{\mathrm{pow}}}(\mu+\varepsilon)\varepsilon
\right].
\end{align}

Since \(g_{N_{\mathrm{pow}}}\) is bounded and
\begin{align}
\mathbb E\|\varepsilon\|_{\mathcal H}^2
=
\sigma^2\operatorname{Tr}(\Sigma)
<
\infty,
\end{align}
the Bochner integral defining \(m(\mu)\) is well-defined in \(\mathcal H\).

Using the Karhunen--Lo\`eve expansion of \(\varepsilon\),
\begin{align}
\varepsilon
=
\sigma
\sum_{k:\lambda_k>0}
\sqrt{\lambda_k}\,\xi_k e_k,
\end{align}
we compute
\begin{align}
m(\mu)
=
\sigma
\sum_{k:\lambda_k>0}
\sqrt{\lambda_k}\,
\mathbb E\!\left[
g_{N_{\mathrm{pow}}}(\mu+\varepsilon)\xi_k
\right]
e_k.
\end{align}

Hence
\begin{align}
\frac{\langle m(\mu),e_k\rangle_{\mathcal H}}
{\sigma\sqrt{\lambda_k}}
=
\mathbb E\!\left[
g_{N_{\mathrm{pow}}}(\mu+\varepsilon)\xi_k
\right].
\end{align}

Therefore, for \(h\in\mathcal H_{\Sigma}\),
\begin{align}
\langle m(\mu),h\rangle_{\mathcal H_{\Sigma}}
&=
\sum_{k:\lambda_k>0}
\frac{
\langle m(\mu),e_k\rangle_{\mathcal H}
\langle h,e_k\rangle_{\mathcal H}
}{
\sigma^2\lambda_k
}
\\
&=
\sum_{k:\lambda_k>0}
\mathbb E\!\left[
g_{N_{\mathrm{pow}}}(\mu+\varepsilon)\xi_k
\right]
\frac{
\langle h,e_k\rangle_{\mathcal H}
}{
\sigma\sqrt{\lambda_k}
}
\\
&=
\mathbb E\!\left[
g_{N_{\mathrm{pow}}}(\mu+\varepsilon)
\sum_{k:\lambda_k>0}
\frac{
\langle h,e_k\rangle_{\mathcal H}
}{
\sigma\sqrt{\lambda_k}
}
\xi_k
\right]
\\
&=
D_hF_{N_{\mathrm{pow}}}(\mu).
\end{align}

Thus \(m(\mu)\) represents the bounded functional \(\Lambda_\mu\). By uniqueness
of the Riesz representer,
\begin{align}
G(\mu)=m(\mu).
\end{align}

Hence the gradient with respect to
\(\mathcal H_\Sigma\) is rescaled by \(\sigma^{-2}\):
\begin{align}
\nabla_{\mathcal H_\Sigma}
F_{N_{\mathrm{pow}}}(\mu)
=
\sigma^{-2}
\mathbb E\!\left[
g_{N_{\mathrm{pow}}}(\mu+\varepsilon)\varepsilon
\right].
\end{align}

This proves the claim.
\end{proof}
\begin{theorem}[Directional ascent toward optimal point]
\label{app:thm:ascent}
Assume that \(g_N(x):=e^{Nf(x)}\mathbf 1_{\{x\in S\}}\) is bounded for each \(N\).
For \(M>0\), define the ellipsoid \(K_M:=\overline{\left\{\mu\in\mathcal H_{\Sigma}:\|\mu\|_{\mathcal H_{\Sigma}}\le M\right\}}.\)
Then, for every \(\epsilon>0\) and \(M>0\), there exists
\(N_0=N_0(\epsilon,M,\sigma,\Sigma)<\infty\) such that for all \(N\ge N_0\), for every \(\mu\in K_M\) satisfying $\|\mu-\xi^\star\|_{\mathcal H}\ge\epsilon$,
letting
\begin{align}
u:=\frac{\xi^\star-\mu}{\|\xi^\star-\mu\|_{\mathcal H}},    
\end{align}
we have $D_{\Sigma u}F_{N_{\mathrm{pow}}}(\mu)>0$.
\end{theorem}

\begin{proof}
Fix \(\epsilon>0\) and \(M>0\). Since \(\Sigma\) is trace class,
\(\Sigma^{1/2}\) is compact. Therefore
\begin{align}
K_M
=
\overline{\{v\in\mathcal H_{\Sigma}:\|v\|_{\mathcal H_{\Sigma}}\le M\}}
\end{align}
is compact in the \(\mathcal H\)-topology.

Choose \(\delta\in(0,\epsilon/2)\) small enough so that
\begin{align}
A:=S\cap B(\xi^\star,\delta)
\end{align}
contains a nonempty open ball. This is possible because
\(\xi^\star\in\operatorname{int}(S)\). Define
\begin{align}
V_\delta
:=
\sup_{\substack{x\in S\\ \|x-\xi^\star\|_{\mathcal H}\ge \delta}}
f(x).
\end{align}
By assumption of the uniqueness of $\xi^\star$,
\begin{align}
V_\delta<f(\xi^\star).
\end{align}
Choose \(D\) such that
\begin{align}
V_\delta<D<f(\xi^\star).
\end{align}
By continuity of \(f\) at \(\xi^\star\), shrinking \(\delta\) if necessary, we may assume that
\begin{align}
f(x)\ge D
\qquad
\text{for all }x\in A.
\end{align}

Fix \(\mu\in K_M\) satisfying
\begin{align}
\|\mu-\xi^\star\|_{\mathcal H}\ge\epsilon,
\end{align}
and set
\begin{align}
u
=
\frac{\xi^\star-\mu}{\|\xi^\star-\mu\|_{\mathcal H}}.
\end{align}
Then \(u\in\mathcal H\) and \(\|u\|_{\mathcal H}=1\). Moreover,
\begin{align}
\Sigma u\in \operatorname{Range}(\Sigma)\subset \operatorname{Range}(\Sigma^{1/2})
=
\mathcal H_\Sigma.
\end{align}

By Lemma~\ref{app:lem:cm_directional}, applied with \(h:=\Sigma u\),
\begin{align}
D_{\Sigma u}F_{N_{\mathrm{pow}}}(\mu)
=
\frac{1}{\sigma^2}\mathbb E\!\left[
e^{Nf(\mu+\varepsilon)}
\mathbf 1_{\{\mu+\varepsilon\in S\}}
\widehat{h}(\varepsilon)
\right].
\end{align}
Let
\begin{align}
\varepsilon
=
\sigma
\sum_{k:\lambda_k>0}
\sqrt{\lambda_k}\,\xi_k e_k.
\end{align}
Since
\begin{align}
\langle \Sigma u,e_k\rangle_{\mathcal H}
=
\lambda_k\langle u,e_k\rangle_{\mathcal H},
\end{align}
we have,
\begin{align}
\langle \varepsilon,u\rangle_{\mathcal H}
=
\sigma
\sum_{k:\lambda_k>0}
\sqrt{\lambda_k}\langle u,e_k\rangle_{\mathcal H}\xi_k.
\end{align}
Therefore
\begin{align}
\widehat{h}(\varepsilon)=\langle \varepsilon,u\rangle_{\mathcal H}.
\end{align}
Thus, with
\begin{align}
X:=\mu+\varepsilon,
\end{align}
finally, we obtain
\begin{equation}
D_{\Sigma u}F_{N_{\mathrm{pow}}}(\mu)
=
\frac{1}{\sigma^2}
\mathbb E\!\left[
e^{Nf(X)}
\mathbf 1_{\{X\in S\}}
\langle X-\mu,u\rangle_{\mathcal H}
\right].
\label{app:eq:Dh-sigmau-correct}
\end{equation}

We decompose the expectation in \eqref{app:eq:Dh-sigmau-correct} into
\begin{align}
T_1
:=
\mathbb E\!\left[
e^{Nf(X)}
\langle X-\mu,u\rangle_{\mathcal H}
\mathbf 1_{\{X\in A\}}
\right],
\end{align}
and
\begin{align}
T_2
:=
\mathbb E\!\left[
e^{Nf(X)}
\langle X-\mu,u\rangle_{\mathcal H}
\mathbf 1_{\{X\in S\setminus A\}}
\right].
\end{align}

If \(X\in A\), then
\begin{align}
\|X-\xi^\star\|_{\mathcal H}\le\delta.
\end{align}
Hence
\begin{align}
\begin{aligned}
\langle X-\mu,u\rangle_{\mathcal H}
&=
\langle X-\xi^\star,u\rangle_{\mathcal H}
+
\langle \xi^\star-\mu,u\rangle_{\mathcal H}  \\
&\ge
-\|X-\xi^\star\|_{\mathcal H}
+
\|\xi^\star-\mu\|_{\mathcal H} \\
&\ge
\epsilon-\delta
>
0.
\end{aligned}
\end{align}
Since \(f(X)\ge D\) on \(A\), we get
\begin{align}
T_1
\ge
(\epsilon-\delta)e^{ND}
\mathbb P_\mu(X\in A),
\end{align}
where \(X\sim\mathcal N(\mu,\sigma^2\Sigma)\).

Since \(A\) contains a nonempty open set and the Gaussian measure
\(\mathcal N(0,\sigma^2\Sigma)\) has full support, we have
\begin{align}
\mathbb P_\mu(X\in A)>0
\qquad
\text{for each }\mu\in K_M.
\end{align}
Moreover, the map
\begin{align}
\mu\mapsto \mathbb P_\mu(X\in A)
\end{align}
is continuous on \(K_M\). Since \(K_M\) is compact,
\begin{align}
p_A
:=
\inf_{\mu\in K_M}
\mathbb P_\mu(X\in A)
>
0.
\end{align}
Therefore
\begin{equation}
T_1
\ge
(\epsilon-\delta)p_A e^{ND}.
\label{app:eq:T1-final-correct}
\end{equation}

Next, on \(S\setminus A\), we have
\begin{align}
f(X)\le V_\delta.
\end{align}
Therefore
\begin{align}
|T_2|
\le
e^{NV_\delta}
\mathbb E\!\left[
|\langle X-\mu,u\rangle_{\mathcal H}|
\right].
\end{align}
Since
\begin{align}
X-\mu=\varepsilon,
\end{align}
and \(\langle \varepsilon,u\rangle_{\mathcal H}\) is centered Gaussian with variance
\begin{align}
\sigma^2\langle \Sigma u,u\rangle_{\mathcal H}
\le
\sigma^2\|\Sigma\|_{\mathrm{op}},
\end{align}
we obtain
\begin{align}
\mathbb E\!\left[
|\langle \varepsilon,u\rangle_{\mathcal H}|
\right]
\le
\sigma
\sqrt{\frac{2\|\Sigma\|_{\mathrm{op}}}{\pi}}.
\end{align}
Hence
\begin{equation}
|T_2|
\le
\sigma
\sqrt{\frac{2\|\Sigma\|_{\mathrm{op}}}{\pi}}
e^{NV_\delta}.
\label{app:eq:T2-final-correct}
\end{equation}

Since \(D>V_\delta\), there exists \(N_0<\infty\) such that for all
\(N\ge N_0\),
\begin{align}
\sigma
\sqrt{\frac{2\|\Sigma\|_{\mathrm{op}}}{\pi}}
e^{NV_\delta}
\le
\frac12
(\epsilon-\delta)p_Ae^{ND}.
\end{align}
Combining \eqref{app:eq:T1-final-correct} and \eqref{app:eq:T2-final-correct}, we get
\begin{align}
T_1+T_2
\ge
T_1-|T_2|
\ge
\frac12
(\epsilon-\delta)p_Ae^{ND}
>
0.
\end{align}
Substituting this into \eqref{app:eq:Dh-sigmau-correct} yields
\begin{align}
D_{\Sigma u}F_{N_{\mathrm{pow}}}(\mu)
=
\frac{1}{\sigma^2}(T_1+T_2)
>
0.
\end{align}
This holds uniformly for every
\begin{align}
\mu\in K_M
\qquad\text{with}\qquad
\|\mu-\xi^\star\|_{\mathcal H}\ge\epsilon.
\end{align}
The proof is complete.
\end{proof}

\begin{lemma}[Second directional derivative identity]
\label{app:lem:second-directional}
Assume that 
\(g_{N_{\mathrm{pow}}}\) is bounded. Then, for any
\(h,k\in\mathcal H_\Sigma=\operatorname{Range}(\Sigma^{1/2})\), the mixed
second directional derivative exists and satisfies
\begin{align}
D^2F_{N_{\mathrm{pow}}}(\mu)[h,k]
=
\frac{1}{\sigma^{2}}\mathbb E\!\left[
g_{N_{\mathrm{pow}}}(\mu+\varepsilon)
\left(
\widehat h(\varepsilon)
\widehat k(\varepsilon)
-
\langle h,k\rangle_{\mathcal H_{\Sigma}}
\right)
\right],
\label{app:eq:second-directional-cm}
\end{align}
where
\begin{align}
\widehat h(\varepsilon)
:=
\sum_{j:\lambda_j>0}
\frac{\langle h,e_j\rangle_{\mathcal H}}
{\sqrt{\lambda_j}}
\xi_j,
\qquad
\widehat k(\varepsilon)
:=
\sum_{j:\lambda_j>0}
\frac{\langle k,e_j\rangle_{\mathcal H}}
{\sqrt{\lambda_j}}
\xi_j,
\end{align}
and \(D^2F_{N_{\mathrm{pow}}}(\mu)[h,k]=D^2F_{N_{\mathrm{pow}}}(\mu)[k,h].\)
\end{lemma}

\begin{proof}
Fix \(\mu\in\mathcal H\) and \(h,k\in\mathcal H_\Sigma\).
For \((s,t)\in\mathbb R^2\), define
\begin{align}
\Phi(s,t)
:=
F_{N_{\mathrm{pow}}}(\mu+s h+t k)
=
\int_{\mathcal H}
g_{N_{\mathrm{pow}}}(\mu+s h+t k+x)\,\gamma(dx).
\end{align}
Using the change of variables \(y=x+s h+t k\), we write
\begin{align}
\Phi(s,t)
=
\int_{\mathcal H}
g_{N_{\mathrm{pow}}}(\mu+y)\,\gamma_{s,t}(dy),
\end{align}
where
\begin{align}
\gamma_{s,t}(A)
=
\gamma(A-s h-t k).
\end{align}

Since \(s h+t k\in\mathcal H_{\Sigma}\), the Cameron--Martin theorem gives
\(\gamma_{s,t}\ll\gamma\), with Radon--Nikodym derivative
\begin{align}
\frac{d\gamma_{s,t}}{d\gamma}(y)
=
\exp\!\left(
\frac{1}{\sigma}s\widehat h(y)
+
\frac{1}{\sigma}t\widehat k(y)
-
\frac{1}{2\sigma^{2}}
\|s h+t k\|_{\mathcal H_{\Sigma}}^2
\right).
\end{align}
Denote this density by \(R_{s,t}(y)\). Then
\begin{align}
\Phi(s,t)
=
\mathbb E\!\left[
g_{N_{\mathrm{pow}}}(\mu+\varepsilon)
R_{s,t}(\varepsilon)
\right].
\end{align}

By the Karhunen--Lo\`eve representation,
\begin{align}
\widehat h(\varepsilon)
=
\sum_{j:\lambda_j>0}
\frac{\langle h,e_j\rangle_{\mathcal H}}
{\sigma\sqrt{\lambda_j}}
\xi_j,
\qquad
\widehat k(\varepsilon)
=
\sum_{j:\lambda_j>0}
\frac{\langle k,e_j\rangle_{\mathcal H}}
{\sigma\sqrt{\lambda_j}}
\xi_j.
\end{align}
These are centered Gaussian random variables with finite variances
\begin{align}
\mathbb E|\widehat h(\varepsilon)|^2
=
\|h\|_{\mathcal H_{\Sigma}}^2
<\infty,
\qquad
\mathbb E|\widehat k(\varepsilon)|^2
=
\|k\|_{\mathcal H_{\Sigma}}^2
<\infty.
\end{align}
Moreover,
\begin{align}
\mathbb E[
\widehat h(\varepsilon)
\widehat k(\varepsilon)
]
=
\langle h,k\rangle_{\mathcal H_{\Sigma}}.
\end{align}

Expanding the Cameron--Martin norm gives
\begin{align}
\|s h+t k\|_{\mathcal H_{\Sigma}}^2
=
s^2\|h\|_{\mathcal H_{\Sigma}}^2
+
2st\langle h,k\rangle_{\mathcal H_{\Sigma}}
+
t^2\|k\|_{\mathcal H_{\Sigma}}^2.
\end{align}
Therefore,
\begin{align}
R_{s,t}(\varepsilon)
=
\exp\!\left(
\frac{s}{\sigma}\widehat h(\varepsilon)
+
\frac{t}{\sigma}\widehat k(\varepsilon)
-
\frac{1}{2\sigma^{2}}s^2\|h\|_{\mathcal H_{\Sigma}}^2
-
\frac{1}{\sigma^{2}}st\langle h,k\rangle_{\mathcal H_{\Sigma}}
-
\frac{1}{2\sigma^{2}}t^2\|k\|_{\mathcal H_{\Sigma}}^2
\right).
\end{align}

Since \(g_{N_{\mathrm{pow}}}\) is bounded and \(R_{s,t}\) is the exponential of a
two-dimensional Gaussian random vector, differentiation under the expectation is
justified locally around \((s,t)=(0,0)\). A direct differentiation gives
\begin{align}
\left.
\frac{\partial^2}{\partial s\,\partial t}
R_{s,t}(\varepsilon)
\right|_{(s,t)=(0,0)}
=
\frac{1}{\sigma^{2}}\left(\widehat h(\varepsilon)
\widehat k(\varepsilon)
-
\langle h,k\rangle_{\mathcal H_{\Sigma}}\right).
\end{align}
Hence
\begin{align}
D^2F_{N_{\mathrm{pow}}}(\mu)[h,k]
=
\partial_s\partial_t\Phi(0,0)
=
\frac{1}{\sigma^{2}}\mathbb E\!\left[
g_{N_{\mathrm{pow}}}(\mu+\varepsilon)
\left(
\widehat h(\varepsilon)
\widehat k(\varepsilon)
-
\langle h,k\rangle_{\mathcal H_{\Sigma}}
\right)
\right].
\end{align}
This is exactly
\eqref{app:eq:second-directional-cm}. Symmetry follows because
\begin{align}
\widehat h(\varepsilon)\widehat k(\varepsilon)
=
\widehat k(\varepsilon)\widehat h(\varepsilon)
\end{align}
and
\begin{align}
\langle h,k\rangle_{\mathcal H_{\Sigma}}
=
\langle k,h\rangle_{\mathcal H_{\Sigma}}.
\end{align}
The proof is complete.
\end{proof}

\begin{lemma}[Lipschitz continuity of the natural functional gradient]
\label{app:lem:lipschitz-from-second-derivative}
Assume that \(g_{N_{\mathrm{pow}}}\) is bounded, and define $\|g_{N_{\mathrm{pow}}}\|_\infty:=\sup_{\xi\in\mathcal H}|g_{N_{\mathrm{pow}}}(\xi)|<\infty$
Then the natural functional gradient
$\nabla_{\mathcal H_{\Sigma}}F_{N_{\mathrm{pow}}}$ is Lipschitz continuous along Cameron--Martin directions: for any $\mu,\nu\in\mathcal H$ such that $\mu-\nu\in\mathcal H_{\Sigma}$, we have
\begin{align}
\left\|
\nabla_{\mathcal H_\Sigma}
F_{N_{\mathrm{pow}}}(\mu)
-
\nabla_{\mathcal H_\Sigma}
F_{N_{\mathrm{pow}}}(\nu)
\right\|_{\mathcal H_\Sigma}
\le
\frac{2\|g_{N_{\mathrm{pow}}}\|_\infty}{\sigma^2}
\,
\|\mu-\nu\|_{\mathcal H_\Sigma}.
\end{align}
\end{lemma}

\begin{proof}
Fix \(\mu\in\mathcal H\) and \(h,k\in\mathcal H_{\Sigma}\).
By Lemma~\ref{app:lem:second-directional},
\begin{align}
D^2F_{N_{\mathrm{pow}}}(\mu)[h,k]
=
\frac{1}{\sigma^{2}}\mathbb E\!\left[
g_{N_{\mathrm{pow}}}(\mu+\varepsilon)
\left(
\widehat h(\varepsilon)
\widehat k(\varepsilon)
-
\langle h,k\rangle_{\mathcal H_{\Sigma}}
\right)
\right].
\end{align}

The Gaussian random variables
\begin{align}
\widehat h(\varepsilon)
=
\sum_{j:\lambda_j>0}
\frac{\langle h,e_j\rangle_{\mathcal H}}
{\sqrt{\lambda_j}}
\xi_j,
\qquad
\widehat k(\varepsilon)
=
\sum_{j:\lambda_j>0}
\frac{\langle k,e_j\rangle_{\mathcal H}}
{\sqrt{\lambda_j}}
\xi_j
\end{align}
are centered with variances
\begin{align}
\mathbb E|\widehat h(\varepsilon)|^2
=
\|h\|_{\mathcal H_{\Sigma}}^2,
\qquad
\mathbb E|\widehat k(\varepsilon)|^2
=
\|k\|_{\mathcal H_{\Sigma}}^2.
\end{align}
Therefore, by Cauchy--Schwarz,
\begin{align}
\mathbb E\!\left[
|\widehat h(\varepsilon)\widehat k(\varepsilon)|
\right]
\le
\|h\|_{\mathcal H_{\Sigma}}\|k\|_{\mathcal H_{\Sigma}}.
\end{align}

Using boundedness of \(g_{N_{\mathrm{pow}}}\), we obtain
\begin{align}
|D^2F_{N_{\mathrm{pow}}}(\mu)[h,k]|
&\le
\frac{1}{\sigma^{2}}\|g_{N_{\mathrm{pow}}}\|_\infty
\mathbb E\!\left[
|\widehat h(\varepsilon)\widehat k(\varepsilon)|
\right]
+
\frac{1}{\sigma^{2}}\|g_{N_{\mathrm{pow}}}\|_\infty
|\langle h,k\rangle_{\mathcal H_{\Sigma}}|
\\
&\le
\frac{2}{\sigma^{2}}
\|g_{N_{\mathrm{pow}}}\|_\infty
\|h\|_{\mathcal H_{\Sigma}}\|k\|_{\mathcal H_{\Sigma}}.
\end{align}

Now fix \(\mu,\nu\in\mathcal H\) such that
\begin{align}
d:=\nu-\mu\in\mathcal H_{\Sigma}.
\end{align}
For \(t\in[0,1]\), define
\begin{align}
\psi(t)
:=
D_hF_{N_{\mathrm{pow}}}(\mu+t d).
\end{align}
Then
\begin{align}
\psi'(t)
=
D^2F_{N_{\mathrm{pow}}}(\mu+t d)[h,d].
\end{align}
Hence,
\begin{align}
D_hF_{N_{\mathrm{pow}}}(\nu)
-
D_hF_{N_{\mathrm{pow}}}(\mu)
=
\int_0^1
D^2F_{N_{\mathrm{pow}}}(\mu+t d)[h,d]
\,dt.
\end{align}
Taking absolute values and applying the previous estimate yields
\begin{align}
|D_hF_{N_{\mathrm{pow}}}(\nu)-D_hF_{N_{\mathrm{pow}}}(\mu)|
&\le
\int_0^1
\frac{2}{\sigma^{2}}\|g_{N_{\mathrm{pow}}}\|_\infty
\|h\|_{\mathcal H_{\Sigma}}
\|d\|_{\mathcal H_{\Sigma}}
\,dt
\\
&=
\frac{2}{\sigma^{2}}\|g_{N_{\mathrm{pow}}}\|_\infty
\|h\|_{\mathcal H_{\Sigma}}
\|\nu-\mu\|_{\mathcal H_{\Sigma}}.
\end{align}

By Lemma~\ref{app:lem:natgrad_closed},
\begin{align}
D_hF_{N_{\mathrm{pow}}}(\mu)
=
\left\langle
\nabla_{\mathcal H_{\Sigma}}F_{N_{\mathrm{pow}}}(\mu),
h
\right\rangle_{\mathcal H_{\Sigma}}.
\end{align}
Therefore,
\begin{align}
\left\|
\nabla_{\mathcal H_\Sigma}F_{N_{\mathrm{pow}}}(\nu)
-
\nabla_{\mathcal H_\Sigma}F_{N_{\mathrm{pow}}}(\mu)
\right\|_{\mathcal H_\Sigma}
\le
\frac{2\|g_{N_{\mathrm{pow}}}\|_\infty}{\sigma^2}
\|\nu-\mu\|_{\mathcal H_\Sigma}.
\end{align}

This completes the proof.
\end{proof}

\begin{lemma}[Second moment of the projected Monte--Carlo natural gradient estimator]
\label{app:lem:mc-second-moment}
Let \((e_k)_{k\ge1}\) be an eigenbasis of \(\Sigma\), and define
\(\mathcal H_{d_{\mathrm{eff}}}:=\operatorname{span}\{e_1,\dots,e_m\}.\)
Let \(P_{d_\mathrm{eff}}:\mathcal H\to\mathcal H_{d_{\mathrm{eff}}}\) be the \(\mathcal H\)-orthogonal projection. Define the projected Monte--Carlo estimator
\begin{align}
\widehat g(\mu) := \frac{1}{B\sigma^{2}} \sum_{b=1}^{B} g_{N_{\mathrm{pow}}}(\mu+\varepsilon^{(b)})P_{d_\mathrm{eff}}\varepsilon^{(b)},
\end{align}
where $\varepsilon^{(b)} \overset{i.i.d.}{\sim}\mathcal N(0,\sigma^2\Sigma)$.
Then \(\widehat g(\mu)\) is an unbiased estimator of $P_{d_\mathrm{eff}}\nabla_{\mathcal H_\Sigma}F_{N_{\mathrm{pow}}}(\mu)$, and satisfies
\begin{align}
\mathbb E\!\left[
\big\|
\widehat g(\mu)
\big\|_{\mathcal H_\Sigma}^2
\right]
\le
\frac{
\|g_{N_{\mathrm{pow}}}\|_\infty^2 d_{\mathrm{eff}}
}{
B\sigma^{2}
}.
\end{align}
\end{lemma}

\begin{proof}
Unbiasedness is immediate by linearity of expectation.
Since
\begin{align}
P_{d_\mathrm{eff}}\varepsilon_t^{(b)}
=
\sigma\sum_{k=1}^{d_{\mathrm{eff}}}
\sqrt{\lambda_k}\xi_{t,k}^{(b)}e_k,
\end{align}
we have
\begin{align}
\|P_{d_\mathrm{eff}}\varepsilon_t^{(b)}\|_{\mathcal H_\Sigma}^2
=
\sigma^2
\sum_{k=1}^{d_{\mathrm{eff}}}
(\xi_{t,k}^{(b)})^2.
\end{align}
Therefore
\begin{align}
\mathbb E
\|P_{d_\mathrm{eff}}\varepsilon_t^{(b)}\|_{\mathcal H_\Sigma}^2
=
\sigma^2d_{\mathrm{eff}}.
\end{align}
Using independence and Jensen's inequality,
\begin{align}
\mathbb E[
\|\widehat g(\mu_t)\|_{\mathcal H_\Sigma}^2
\mid\mu_t]
&\le
\frac{1}{B\sigma^{4}}
\|g\|_\infty^2
\mathbb E
\|P_{d_\mathrm{eff}}\varepsilon_t\|_{\mathcal H_\Sigma}^2=
\frac{\|g\|_\infty^2d_{\mathrm{eff}}}{B\sigma^{2}}.
\end{align}
\end{proof}

\begin{lemma}[One-step expected ascent]
\label{app:lem:onestep-stochastic}
Let \(P_{d_\mathrm{eff}}:\mathcal H\to\mathcal H_{d_{\mathrm{eff}}}\) be the \(\mathcal H\)-orthogonal projection.
Suppose that \(\{\mu_t\}\) is generated by the projected update rule. Then, conditioning on \(\mu_t\),
\begin{align}
\mathbb E\!\left[
F_{N_{\mathrm{pow}}}(\mu_{t+1})
\,\big|\,
\mu_t
\right]
\ge\;
&
F_{N_{\mathrm{pow}}}(\mu_t)
+
\eta_t
\big\|
P_{d_\mathrm{eff}}\nabla_{\mathcal H_\Sigma}
F_{N_{\mathrm{pow}}}(\mu_t)
\big\|_{\mathcal H_\Sigma}^{2}
-
\frac{
\|g_{N_{\mathrm{pow}}}\|_\infty^3 d_\mathrm{eff}
}{
B\sigma^{4}
}
\eta_t^2 .
\end{align}
\end{lemma}

\begin{proof}
For simplicity, write $F:=F_{N_{\mathrm{pow}}},\; g:=g_{N_{\mathrm{pow}}},\; G_t:=\nabla_{\mathcal H_\Sigma}F(\mu_t)$.
Since $\widehat g(\mu_t)\in\mathcal H_{d_{\mathrm{eff}}}\subset\mathcal H_\Sigma$,
the ascent lemma in the Hilbert space \(\mathcal H_\Sigma\) applies. Namely, for any
\(d\in\mathcal H_\Sigma\),
\begin{align}
F(\mu_t+d)
\ge
F(\mu_t)
+
\langle G_t,d\rangle_{\mathcal H_\Sigma}
-
\frac{\|g_{N_{\mathrm{pow}}}\|_\infty}{\sigma^2}\|d\|_{\mathcal H_\Sigma}^2.
\end{align}
Taking
\begin{align}
d=\eta_t\widehat g(\mu_t),  
\end{align}
we obtain
\begin{align}
F(\mu_{t+1})
\ge
F(\mu_t)
+
\eta_t
\langle G_t,\widehat g(\mu_t)\rangle_{\mathcal H_\Sigma}
-
\frac{\|g_{N_{\mathrm{pow}}}\|_\infty}{\sigma^2}
\eta_t^2
\|\widehat g(\mu_t)\|_{\mathcal H_\Sigma}^2.
\end{align}

Taking conditional expectation given \(\mu_t\),
\begin{align}
\mathbb E[
F(\mu_{t+1})\mid\mu_t]
\ge
&
F(\mu_t)
+
\eta_t
\left\langle
G_t,
\mathbb E[\widehat g(\mu_t)\mid\mu_t]
\right\rangle_{\mathcal H_\Sigma}
-
\frac{\|g_{N_{\mathrm{pow}}}\|_\infty}{\sigma^2}
\eta_t^2
\mathbb E[
\|\widehat g(\mu_t)\|_{\mathcal H_\Sigma}^2
\mid\mu_t].
\end{align}

By linearity of \(P_{d_\mathrm{eff}}\) and unbiasedness of the Monte--Carlo estimator,
\begin{align}
\mathbb E[\widehat g(\mu_t)\mid\mu_t]
=
P_{d_\mathrm{eff}}G_t.
\end{align}
Here \(P_{d_\mathrm{eff}}\) is the \(\mathcal H\)-orthogonal projection onto \(\mathcal H_{d_{\mathrm{eff}}}\). Since
\(\Sigma e_k=\lambda_k e_k\), the same coordinate projection is also orthogonal with
respect to \(\mathcal H_\Sigma\) on \(\mathcal H_\Sigma\). Hence
\begin{align}
\langle G_t,P_{d_\mathrm{eff}}G_t\rangle_{\mathcal H_\Sigma}
=
\|P_{d_\mathrm{eff}}G_t\|_{\mathcal H_\Sigma}^2=\frac{\sigma^2\|g\|_\infty^2d_{\mathrm{eff}}}{B\sigma^{2}}.
\end{align}
Lemma~\ref{app:lem:mc-second-moment} provides an upper bound of the projected estimator. 
Substituting this bound gives the claim.
\end{proof}

\begin{theorem}[Projected iteration complexity in a Hilbert space]
\label{app:thm:complexity}
Let \(\{\mu_t\}_{t\ge 0}\subset\mathcal H\) be generated by the projected update rule with a constant stepsize \(\eta>0\). Then the iterates satisfy
\begin{align}
\min_{0\le t\le T-1}
\mathbb E\!\left[
\big\|
P_{d_\mathrm{eff}}\nabla_{\mathcal H_\Sigma}
F_{N_{\mathrm{pow}}}(\mu_t)
\big\|_{\mathcal H_\Sigma}
\right]
\le
\varepsilon
\end{align}
provided that
\begin{align}
T
\ge
\frac{
4\,
\|g_{N_{\mathrm{pow}}}\|_\infty^{4}\,
d_{\mathrm{eff}}
}{
B\,\sigma^{4}\varepsilon^{4}
}.
\end{align}
\end{theorem}

\begin{proof}
By Lemma~\ref{app:lem:onestep-stochastic},
\begin{align}
\mathbb E\!\left[
F_{N_{\mathrm{pow}}}(\mu_{t+1})
\,\middle|\,
\mu_t
\right]
\ge
&
F_{N_{\mathrm{pow}}}(\mu_t)
+
\eta
\big\|
P_{d_\mathrm{eff}}\nabla_{\mathcal H_\Sigma}
F_{N_{\mathrm{pow}}}(\mu_t)
\big\|_{\mathcal H_\Sigma}^{2}
-
\frac{
\|g_{N_{\mathrm{pow}}}\|_\infty^3 d_{\mathrm{eff}}
}{
B\sigma^{4}
}
\eta^2.
\end{align}

Taking full expectation and summing over \(t=0,\dots,T-1\),
\begin{align}
\eta
\sum_{t=0}^{T-1}
\mathbb E
\Big\|
P_{d_\mathrm{eff}}\nabla_{\mathcal H_\Sigma}
F_{N_{\mathrm{pow}}}(\mu_t)
\Big\|_{\mathcal H_\Sigma}^{2}
\le
&
\mathbb E[F_{N_{\mathrm{pow}}}(\mu_T)]
-
F_{N_{\mathrm{pow}}}(\mu_0)
+
\frac{
\|g_{N_{\mathrm{pow}}}\|_\infty^3 d_{\mathrm{eff}}
}{
B\sigma^{4}
}
\eta^2T.
\end{align}

Since
\[
F_{N_{\mathrm{pow}}}(\mu)
\le
\|g_{N_{\mathrm{pow}}}\|_\infty,
\]
we obtain
\begin{align}
\eta
\sum_{t=0}^{T-1}
\mathbb E
\Big\|
P_{d_\mathrm{eff}}\nabla_{\mathcal H_\Sigma}
F_{N_{\mathrm{pow}}}(\mu_t)
\Big\|_{\mathcal H_\Sigma}^{2}
\le
\|g_{N_{\mathrm{pow}}}\|_\infty
+
\frac{
\|g_{N_{\mathrm{pow}}}\|_\infty^3 d_{\mathrm{eff}}
}{
B\sigma^{4}
}
\eta^2T.
\end{align}

Dividing by \(\eta T\),
\begin{align}
\frac1T
\sum_{t=0}^{T-1}
\mathbb E
\Big\|
P_{d_\mathrm{eff}}\nabla_{\mathcal H_\Sigma}
F_{N_{\mathrm{pow}}}(\mu_t)
\Big\|_{\mathcal H_\Sigma}^{2}
\le
\frac{
\|g_{N_{\mathrm{pow}}}\|_\infty
}{
\eta T
}
+
\frac{
\|g_{N_{\mathrm{pow}}}\|_\infty^3 d_{\mathrm{eff}}
}{
B\sigma^{4}
}
\eta.
\end{align}

Hence
\begin{align}
\min_{0\le t\le T-1}
\mathbb E
\Big\|
P_{d_\mathrm{eff}}\nabla_{\mathcal H_\Sigma}
F_{N_{\mathrm{pow}}}(\mu_t)
\Big\|_{\mathcal H_\Sigma}^{2}
\le
\frac{
\|g_{N_{\mathrm{pow}}}\|_\infty
}{
\eta T
}
+
\frac{
\|g_{N_{\mathrm{pow}}}\|_\infty^3 d_{\mathrm{eff}}
}{
B\sigma^{4}
}
\eta.
\end{align}

Optimizing the right-hand side over \(\eta>0\) gives
\[
\eta
=
\sqrt{
\frac{
B\sigma^{4}
}{
\|g_{N_{\mathrm{pow}}}\|_\infty^2d_{\mathrm{eff}}T
}
}.
\]
Substituting this value yields
\begin{align}
\min_{0\le t\le T-1}
\mathbb E
\Big\|
P_{d_\mathrm{eff}}\nabla_{\mathcal H_\Sigma}
F_{N_{\mathrm{pow}}}(\mu_t)
\Big\|_{\mathcal H_\Sigma}^{2}
\le
\frac{
2\|g_{N_{\mathrm{pow}}}\|_\infty^2
\sqrt{d_{\mathrm{eff}}}
}{
\sigma^{2}\sqrt{B}\sqrt{T}
}.
\end{align}

Finally, Jensen's inequality implies
\begin{align}
\min_{0\le t\le T-1}
\mathbb E
\Big\|
P_{d_\mathrm{eff}}\nabla_{\mathcal H_\Sigma}
F_{N_{\mathrm{pow}}}(\mu_t)
\Big\|_{\mathcal H_\Sigma}
\le
\sqrt{
\frac{
2\|g_{N_{\mathrm{pow}}}\|_\infty^2
\sqrt{d_{\mathrm{eff}}}
}{
\sigma^{2}\sqrt{B}\sqrt{T}
}
}.
\end{align}

Therefore, it suffices that
\begin{align}
\frac{
2\|g_{N_{\mathrm{pow}}}\|_\infty^2
\sqrt{d_{\mathrm{eff}}}
}{
\sigma^{2}\sqrt{B}\sqrt{T}
}
\le
\varepsilon^2,
\end{align}
which is equivalent to
\begin{align}
T
\ge
\frac{
4\,
\|g_{N_{\mathrm{pow}}}\|_\infty^{4}\,
d_{\mathrm{eff}}
}{
\sigma^{4}B\,\varepsilon^{4}
}.
\end{align}
\end{proof}

\section{Synthetic Experiment Analysis}

This appendix provides additional details for the synthetic narrow-passage experiment in Sec.~VI-A.  
Sec.~C.1 specifies the environment geometry and the trajectory scoring function used in the benchmark.  
Sec.~C.2 summarizes the optimization methods and the evaluation protocol for all baselines.  
Sec.~C.3 relates the observed success rates and smoothness statistics to the geometric structure of the feasible region and the update mechanisms of the different methods.

\subsection{Environment and Objective}

We consider a one-dimensional trajectory
\begin{align}
y : [0,1] \rightarrow \mathbb{R},
\end{align}
discretized at $L=100$ time steps (100\,Hz over a 1\,s horizon).  
The environment is defined by four rectangular obstacle regions in the $(t,y)$ plane,
\begin{align}
\mathcal{B}_i = [t_i^0, t_i^1] \times [y_i^0, y_i^1],
\end{align}
with
\begin{align}
\begin{aligned}
\mathcal{B}_1 &: t\in[0.2,0.25],\; y\in[-1.0,4.0],\\
\mathcal{B}_2 &: t\in[0.4,0.6],\; y\in[-2.0,2.0],\\
\mathcal{B}_3 &: t\in[0.7,1.0],\; y\in[0.5,5.0],\\
\mathcal{B}_4 &: t\in[0.7,1.0],\; y\in[-5.0,-0.5].
\end{aligned}
\end{align}
The feasible set is the complement of the union of these boxes.  
This geometry induces a narrow, time-fragmented corridor: at each time interval only a small subset of $y(t)$ is feasible, and admissible regions shift discontinuously across time.

For a discretized trajectory $\mathbf{y}=(y_1,\dots,y_L)$, we define the per-time-step penetration score
\[
s_t =
\min_i
\begin{cases}
-\min(y_t-y_i^0,\; y_i^1-y_t), & (t,y_t)\in\mathcal{B}_i,\\
0, & \text{otherwise}.
\end{cases}
\]
Thus, $s_t\le 0$, and $s_t=0$ indicates that the trajectory is collision-free at time step $t$.

The trajectory score is defined by separating collision and collision-free cases:
\[
f(\mathbf{y}) =
\begin{cases}
\exp\!\left(-\lambda_{\mathrm{jerk}}\,\bar{J}(\mathbf{y})\right),
& \text{if } s_t=0 \;\; \forall t,\\[5pt]
\frac{1}{L}\sum_{t=1}^{L}s_t,
& \text{otherwise},
\end{cases}
\]
where $\lambda_{\mathrm{jerk}}=10^{-4}$ and $\bar{J}(\mathbf{y})$ denotes the average absolute jerk of the trajectory. 
The jerk is computed using a third-order finite-difference operator $A_{\mathrm{jerk}}$:
\[
\mathbf{j}=A_{\mathrm{jerk}}\mathbf{y},
\qquad
\bar{J}(\mathbf{y})
=
\frac{1}{L}
\sum_{t=1}^{L} |j_t|.
\]

For trajectories that intersect obstacles, the score is given by the average penetration penalty and remains non-positive. 
For collision-free trajectories, the exponential bonus assigns higher scores to trajectories with lower average absolute jerk. 
Thus, the objective first prioritizes feasibility and then favors smoother feasible trajectories.

\subsection{Optimization Methods}

We compare the proposed natural functional gradient (NFG) method with representative baseline optimizers. 
Our method directly maximizes the score $f(\mathbf{y})$ defined in Sec.~C.1 using Monte-Carlo estimation with $N=100$ trajectory perturbations per iteration. 
All baseline methods instead solve a minimization problem and aim to reduce collision while encouraging smooth feasible trajectories.

STOMP minimizes the transformed objective $J(\mathbf{y}) = 1 - f(\mathbf{y})$ using stochastic finite-difference updates. 
At each iteration, $N=100$ trajectory perturbations are sampled from the same trajectory covariance used by the proposed method, and the update direction is estimated from exponentially weighted trajectory costs. 
Lower-cost trajectories receive larger update weights through normalized exponential reweighting.

CHOMP also minimizes $J(\mathbf{y}) = 1 - f(\mathbf{y})$, but relies on analytic gradient updates instead of stochastic sampling. 
For collision trajectories, the gradient is computed from the outward normal direction corresponding to the deepest penetration region. 
When the trajectory lies exactly at the midpoint between two penetration boundaries, the update direction is selected randomly between the two outward directions to avoid ambiguous gradient assignments. 
For collision-free trajectories, the objective reduces to an exponential penalty on the average absolute jerk, and gradients are computed from the corresponding finite-difference jerk operator. 
CHOMP updates trajectories deterministically aside from this tie-breaking procedure and does not use trajectory sampling.

MPPI follows a rollout-based optimization formulation. 
Rather than directly optimizing $f(\mathbf{y})$, it minimizes a running cost composed of obstacle penalties and a terminal goal objective:
\[
J = \sum_t J_{\mathrm{obs}}(t) + \sum_t J_{\mathrm{goal}}(t),
\]
where $J_{\mathrm{obs}}$ corresponds to the box penetration penalty and
\[
J_{\mathrm{goal}}(t) = (y_t - y_{\mathrm{goal}})^2
\]
encourages the trajectory to reach the target state. 
Trajectory perturbations are generated using Wiener-process noise, and trajectory updates are computed through exponential reweighting of sampled rollouts.

\begin{figure}[t]
    \centering
    \includegraphics[width=0.99\linewidth]{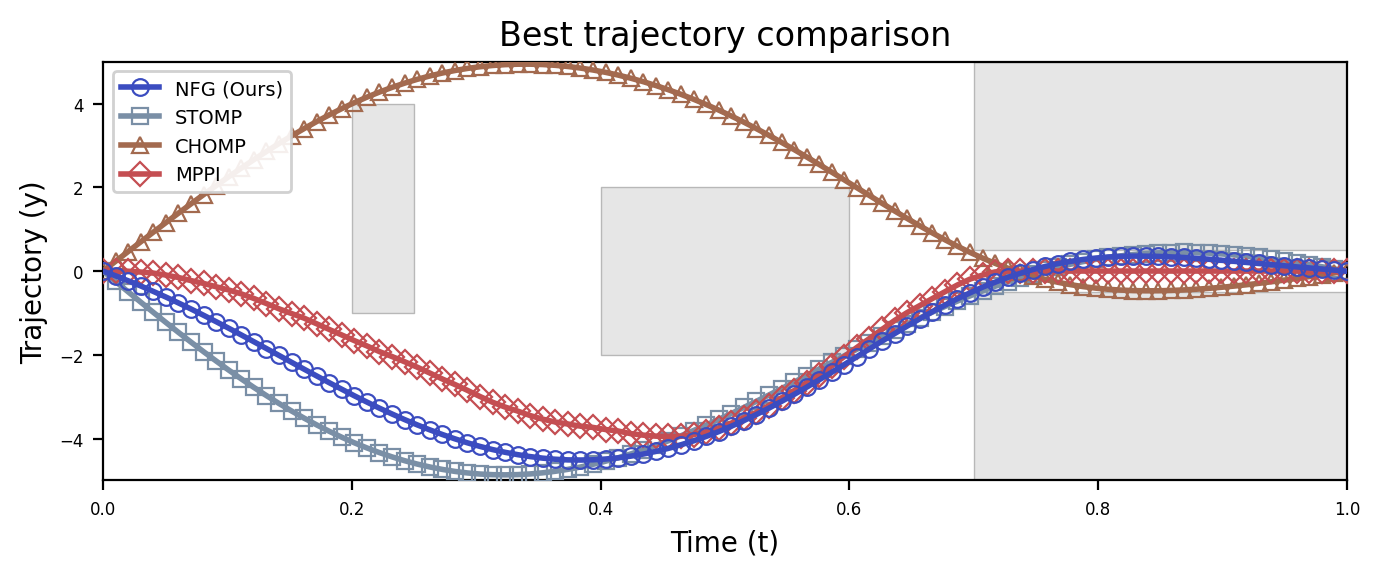}
    \vspace{-3mm}
    \caption{Best trajectories obtained by each method in the synthetic box-avoiding scenario.}
    \label{app:fig:supp_synthetic_traj}
\end{figure}

\subsection{Observed Behavior and Failure Analysis}

All methods are evaluated over five random seeds with identical environment settings. 
The initial trajectory for all runs is a zero-valued trajectory, corresponding to linear interpolation between the initial and final states. 
A trial is counted as successful if the resulting trajectory remains collision-free over the full horizon. 
Reported success rates and smoothness statistics in Table III correspond to averages over these five runs.

The narrow and temporally fragmented feasible corridor occupies only a small region in trajectory space, so most trajectory perturbations intersect at least one obstacle. 
Representative optimized trajectories are illustrated in Fig.~\ref{app:fig:supp_synthetic_traj}. 
STOMP occasionally discovers feasible trajectories, but many sampled perturbations receive similar collision penalties, making consistent progress toward the narrow corridor difficult. 
CHOMP often deforms the trajectory toward the feasible region through local gradient updates, but the optimization can remain sensitive to initialization and local penetration geometry. 
MPPI reaches the feasible corridor more frequently, but the resulting trajectories exhibit substantially larger jerk due to stochastic rollout variability and the absence of an explicit smoothness mechanism.

In contrast, the proposed method performs updates on a Gaussian-smoothed surrogate objective within the induced function-space geometry. 
The covariance structure enables smooth trajectory perturbations while aggregating weak feasibility information across sampled trajectories. 
As a result, the method consistently discovers feasible trajectories that pass through the narrow corridor while maintaining relatively smooth trajectory evolution, yielding the highest success rate and lower jerk than MPPI, as reported in Sec.~VI-A.

\section{Simulation Experiment Analysis}

This appendix provides additional details on the simulation experiments in constrained cabinet environments.
The focus is to clarify the evaluation pipeline used for comparing different planners under a unified execution and collision-checking model.
Since the compared methods produce trajectories with heterogeneous representations and implicit assumptions, all outputs are post-processed into a common time-parameterized form and re-evaluated in a physics-based simulator.

\subsection{Experimental Setup of the Proposed Method}

All simulation experiments are conducted using a Franka Research 3 manipulator equipped with an Inspire left hand. 
The task consists of grasping a cylindrical object from a shelf and inserting it into a target cabinet while avoiding collisions with the surrounding environment. 
The manipulated object is a cylinder with radius $3\,\mathrm{cm}$ and height $30\,\mathrm{cm}$.

The cabinet environment has dimensions of $60\,\mathrm{cm} \times 140\,\mathrm{cm} \times 60\,\mathrm{cm}$, and the cabinet floor is positioned $10\,\mathrm{cm}$ above the table surface. 
Scene difficulty is controlled by varying the cabinet door opening angle, resulting in four environments: free space, fully open, quarter closed, and half closed. 
All environments share the same initial and target robot configurations, while the geometric constraints near the insertion region become progressively tighter as the cabinet opening decreases.

The proposed method generates joint-space trajectories over a fixed discretized horizon and performs optimization directly in function space using Gaussian trajectory perturbations. 
Additional implementation details regarding trajectory discretization, optimization parameters, and execution settings are provided in the supplementary material.

In our implementation, trajectories are represented with a temporal resolution of $100\,\mathrm{Hz}$ over a $5\,\mathrm{s}$ horizon.
At each iteration, $N=100$ candidate trajectories are sampled and evaluated using a trajectory-level objective that reflects geometric feasibility and motion regularity.
All candidates are optimized and compared under this shared discrete-time representation.

Let $\xi^{(j)} = \{\mathbf{q}_t^{(j)}\}_{t=1}^{H}$ denote the $j$-th sampled joint trajectory, where $\mathbf{q}_t^{(j)} \in \mathbb{R}^n$ and $H$ is the planning horizon.
To discourage excessive joint motion, we define a path-length regularization term
\[
C_{\mathrm{len}}^{(j)} = \sum_{d=1}^{n} \sum_{t=1}^{H-1}
\left| q_{t+1,d}^{(j)} - q_{t,d}^{(j)} \right|.
\]

Geometric feasibility is evaluated in the simulator at every time step.
Let $n_t^{(j)}$ denote the number of contact points at time $t$, and let $d_t^{(j)}$ denote the minimum signed distance between the robot and the environment, where $d_t^{(j)} < 0$ indicates penetration.
The instantaneous collision score is defined as
\[
\phi_t^{(j)} =
\begin{cases}
0, & d_t^{(j)} > 0, \\
n_t^{(j)} \, d_t^{(j)}, & d_t^{(j)} \le 0.
\end{cases}
\]

For each trajectory, the collision signal is aggregated using the worst-case value along the horizon,
\[
C_{\mathrm{col}}^{(j)} = \min_{t \in \{1,\dots,H\}} \phi_t^{(j)}.
\]

The final objective for each sampled trajectory is given by
\[
f(\xi^{(j)}) = C_{\mathrm{col}}^{(j)} - \lambda \, C_{\mathrm{len}}^{(j)},
\]
where $\lambda > 0$ balances collision avoidance and joint-space path length.

\subsection{Unified Evaluation Pipeline}

The compared planners produce joint trajectories under different internal representations and execution assumptions.
Some return sparse waypoints, while others optimize time-discretized trajectories or rely on internal smoothing and collision approximations.
To enable a consistent execution-level comparison, all planner outputs are post-processed into a common time-parameterized joint-space trajectory and evaluated using the same simulation and collision-checking pipeline.

Let $\{\mathbf{q}_i\}_{i=1}^{L}$ denote the joint-space waypoints produced by a planner, where $\mathbf{q}_i \in \mathbb{R}^n$.
After unwrapping angular joints to remove $2\pi$ discontinuities, we compute the incremental joint-space distances
\begin{align}
\Delta s_i = \|\mathbf{q}_{i+1} - \mathbf{q}_i\|_2, \quad i = 1, \dots, L-1,
\end{align}
and define the cumulative arc-length
\begin{align}
s_1 = 0, \quad
s_i = \sum_{j=1}^{i-1} \Delta s_j, \quad i = 2, \dots, L,
\end{align}
with total path length $S = s_L$.

Each waypoint is assigned a time stamp by linearly mapping arc-length to a fixed duration $T$,
\begin{align}
t_i = \frac{s_i}{S} T, \quad i = 1, \dots, L.
\end{align}
This defines a continuous-time parameterization of the discrete path that is independent of the planner’s internal time representation.

A uniform time grid is then defined as
\begin{align}
\tau_k = \frac{k-1}{f}, \quad k = 1, \dots, M, \qquad M = Tf,
\end{align}
where $f = 100\,\mathrm{Hz}$ and $T = 5\,\mathrm{s}$.
This grid defines the shared execution resolution used for all planners.

For each joint dimension $d \in \{1, \dots, n\}$, the resampled trajectory is obtained by linear interpolation between consecutive waypoints,
\begin{align}
q_d(\tau_k) = q_{i,d} + \frac{\tau_k - t_i}{t_{i+1}-t_i}
\bigl(q_{i+1,d} - q_{i,d}\bigr),
\quad t_i \le \tau_k < t_{i+1}.
\end{align}
This yields a time-parameterized joint trajectory
\begin{align}
\mathbf{q}(\tau_k) = \bigl(q_1(\tau_k), \dots, q_n(\tau_k)\bigr), \quad k = 1, \dots, M.
\end{align}

The resampled trajectory is executed in MuJoCo, and collision checks are performed at every time step.
A trial is counted as successful if and only if the entire time-parameterized trajectory is collision-free.

Because the compared planners rely on different internal representations, collision models, and time discretizations, their outputs are evaluated under different native execution assumptions.
The unified evaluation therefore provides a common execution-level reference in which all methods are assessed under identical simulation and collision-checking conditions.
Table~\ref{app:tab:traj_representation} summarizes the main structural differences in how each method represents and executes trajectories.

Although CHOMP and STOMP also optimize time-discretized trajectories with the same nominal horizon and sampling rate, their update rules are defined within their own internal execution models.
As a result, some trajectories that are feasible under native settings can exhibit large inter-step joint changes when resampled and executed at a fixed rate.
These effects should be interpreted as differences in how planner-specific assumptions translate under a shared execution model, rather than as a statement of correctness of one execution model over another.

\begin{table}[t]
\vspace{4pt} 
\centering
\small
\setlength{\tabcolsep}{6pt}
\renewcommand{\arraystretch}{1.0}
\caption{Summary of structural differences in trajectory representation and execution among the compared planners.}
\vspace{4pt}
\label{app:tab:traj_representation}
\begin{tabular}{p{0.24\columnwidth} p{0.33\columnwidth} p{0.33\columnwidth}}
\toprule
\textbf{Method} &
\textbf{What is optimized} &
\textbf{How it is executed} \\
\midrule
BiRRT &
Sparse geometric waypoints &
Smoothed and time-parameterized \\

CHOMP &
Time-discretized joint trajectory &
Executed on fixed grid \\

STOMP &
Stochastic time trajectory &
Executed on fixed grid \\

cuRobo &
Convexified approximate motion &
Executed under internal model \\

\textbf{NFG (Ours)} &
Function-space joint trajectory &
Executed under unified model \\
\bottomrule
\end{tabular}
\end{table}

\subsection{Analysis on Baseline Execution Behaviors}
To examine execution-level feasibility under a shared reference model, all planner outputs are re-evaluated using the unified post-processing and collision-checking pipeline described in Section~B.2.
Each output is first converted into a fixed-length, time-parameterized joint trajectory and then executed in MuJoCo, where collision checks are performed at every time step.
In our experimental evaluation, we observed a distinct discrepancy: several baseline methods (cuRobo, STOMP, and CHOMP) frequently reported finding valid, collision-free solutions based on their internal metrics. However, when these same solutions were subjected to the unified evaluation pipeline described above, they often resulted in collisions. We analyzed the generated trajectories and identified that these failures were not necessarily due to algorithmic errors, but rather due to structural differences in geometric approximation and temporal discretization handling.

\subsubsection{BiTRRT}
For sampling-based planners such as BiTRRT, the outputs typically consist of sparse geometric waypoints that connect the start and goal configurations. When these sparse configurations are linearly interpolated to match the 100 Hz shared execution resolution, collisions were occasionally observed along the resampled trajectory segments.

This phenomenon highlights a trade-off inherent in the configuration of the planner. While it is possible to mitigate execution-level collisions by reducing the planning range (i.e., the maximum length of edges in the tree) or increasing the resolution of edge collision checking, such adjustments generally increase the computational burden. In our experiments, stricter edge constraints often led to increased planning times or a higher frequency of timeouts, resulting in incomplete solutions. This indicates that performance under the unified evaluation is highly dependent on the balance between planning efficiency and the granularity of the generated path.

\begin{figure}[h]
    \centering
    \includegraphics[width=0.99\linewidth]{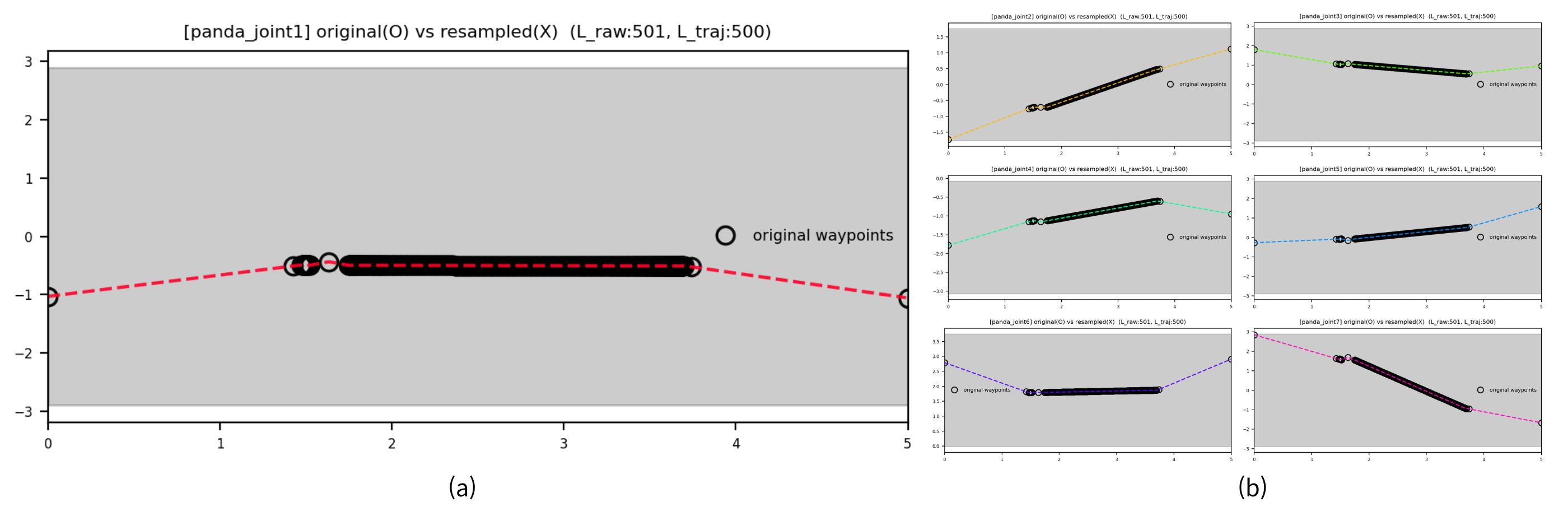}
    \vspace{-3mm}
    \caption{Trajectory comparison for the 7-DOF Panda manipulator. The sequences of black circles represent the discrete, original waypoints optimized by CHOMP, while the red dotted lines represent the linearly interpolated trajectory used for execution. A significant discontinuity is observed between the initial configuration ($t=0$) and the first optimized waypoint, creating a linear path that may violate collision constraints during continuous execution.}
    \label{app:fig:stomp_fail}
\end{figure}

\begin{figure}[h]
    \centering
    \includegraphics[width=0.99\linewidth]{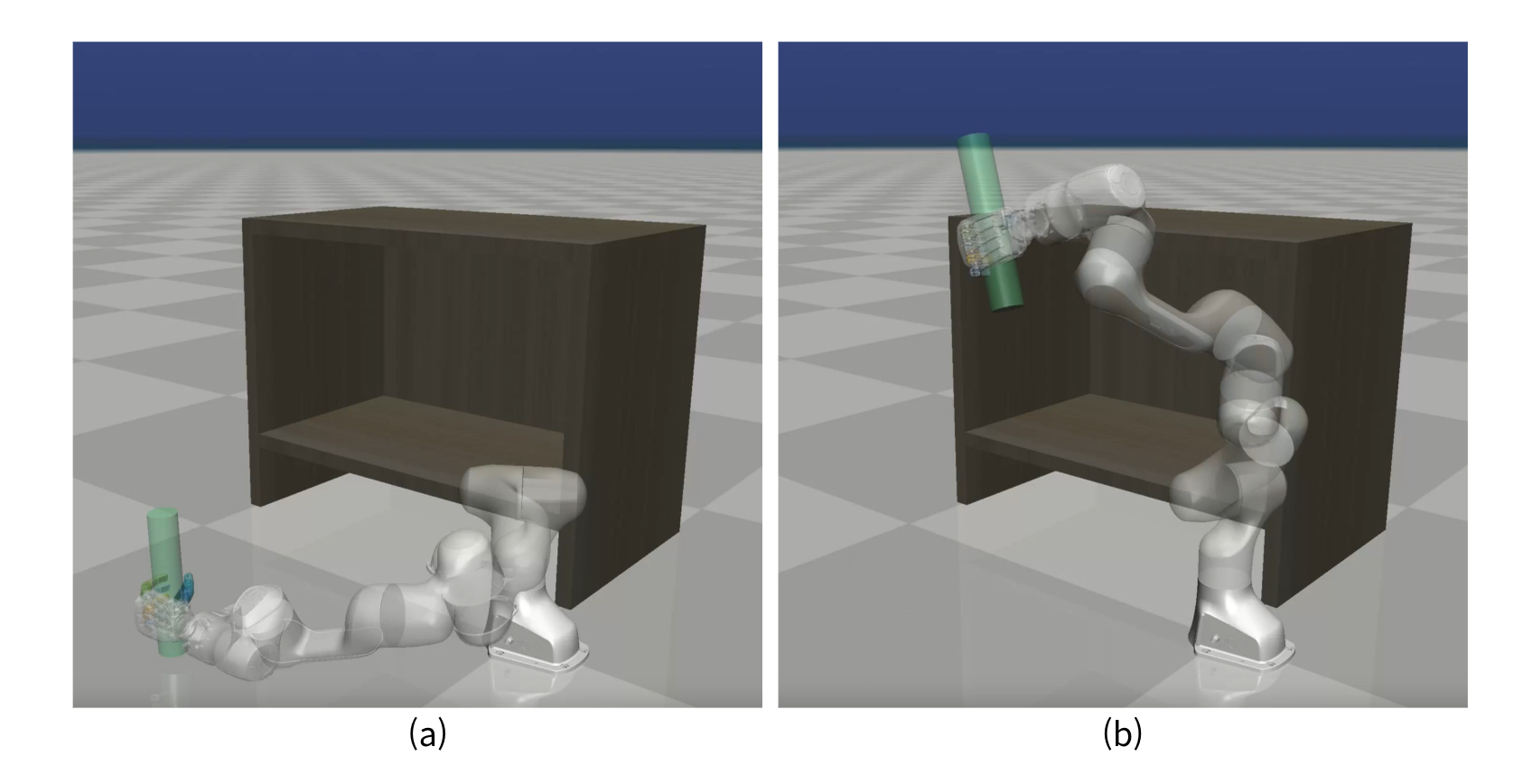}
    \vspace{-3mm}
    \caption{Visualization of the kinematic discontinuity in the MuJoCo environment. (a) The initial robot configuration and (b) the immediate next waypoint generated by CHOMP. The large spatial displacement between these two consecutive states implies that the linearly interpolated path—required for controller execution—inevitably intersects with the cabinet geometry, resulting in a collision despite both individual endpoints being collision-free.}
    \label{app:fig:stomp_sim_fail}
\end{figure}

\subsubsection{CHOMP and STOMP}

For discrete-time optimization methods such as CHOMP and STOMP, the discrepancies observed under the unified evaluation were primarily driven by kinematic inconsistencies arising from the discretization of the trajectory. These methods optimize a fixed set of waypoints to minimize a cost function.

Our analysis revealed that while the optimized waypoints themselves were mostly collision-free, the solvers can exhibit discontinuities under the unified execution model suitable for high-frequency control. As illustrated in the supplementary figures (e.g., Fig.~\ref{app:fig:stomp_fail}), we observed significant joint-space jumps, particularly between the fixed start configuration and the first adjacent optimized waypoint (and similarly near the goal configuration). Fig.~\ref{app:fig:stomp_sim_fail} visualizes the physical consequence of this discontinuity: the robot is required to transition between two vastly different configurations—from a resting pose to an insertion pose—within a single discrete time step.

Although the discrete states at steps $t$ and $t+1$ (e.g., the start state and the first waypoint) were valid configurations in isolation, they were sufficiently distant in configuration space that the linear path connecting them intersected with obstacles. When Eq. (143) was applied to interpolate these sparse waypoints for the 100 Hz controller, the intermediate states generated during this transition violated collision constraints. 

This phenomenon indicates that checking feasibility only at discrete knot points may be insufficient for complex manipulation tasks where obstacles lie in the manifold between waypoints. Consequently, a divergence occurs between the feasibility under native planner assumptions (based on discrete knot points) and the actual execution feasibility (which requires continuous motion through the interpolated path).

\subsubsection{cuRobo}

For cuRobo, we utilized the built-in functionality to approximate the attached cylindrical object and robot links using collision spheres. While this approximation is efficient for collision checking within the planner's native framework, it introduces a geometric discrepancy compared to the high-fidelity meshes used in the MuJoCo simulator.

This discrepancy is primarily associated with the configuration of the primitive-based approximation. As illustrated in the supplementary figures, the method relies on pre-defined collision spheres (e.g., the yellow spheres in Fig.~\ref{app:fig:curobo_fail}(a)) to represent the robot's volume. Consequently, the planner's ability to navigate tight spaces becomes sensitive to the size and distribution of these primitives.

If the collision spheres are set to be smaller or sparsely distributed, the planner may not detect potential collisions in the interstitial spaces between spheres or in the gaps between the primitive approximation and the actual mesh surface (as shown in Fig.~\ref{app:fig:curobo_fail}(b)). This can lead to trajectories that are valid under the sphere model but result in contact during high-fidelity execution. Conversely, if the spheres are inflated to cover these gaps, the effective volume of the robot increases. In the "half-closed" cabinet scenario, this expanded geometry can occupy the limited available clearance, potentially preventing the planner from finding a feasible solution due to the restricted free space.

While increasing the number of collision spheres could reduce these geometric gaps, doing so generally increases the computational cost of the collision checking routines. This presents a trade-off where improving geometric accuracy to match the unified evaluation standard may affect the computational efficiency that characterizes the method.

\begin{figure}[h]
    \centering
    \includegraphics[width=0.99\linewidth]{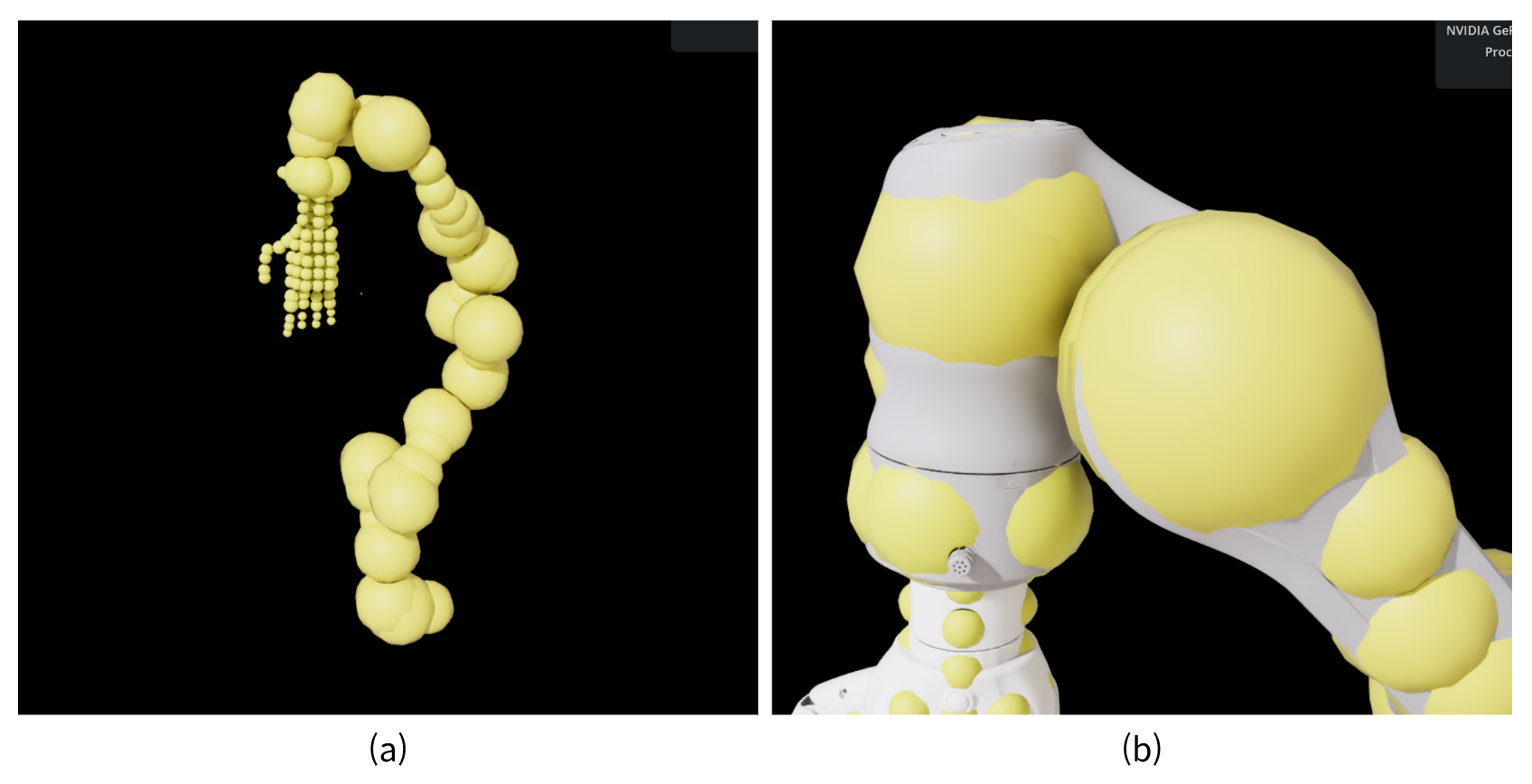}
    \vspace{-3mm}
    \caption{(a) Illustration of the user-defined collision spheres (yellow) approximating the robot links. (b) Comparison showing the volumetric mismatch between the collision primitives and the original mesh. Reducing the sphere size or density results in interstitial gaps where collisions fail to be detected. Conversely, inflating the spheres to cover these gaps introduces excessive padding, which artificially restricts robot movement and can prevent the discovery of feasible trajectories in constrained environments.}
    \label{app:fig:curobo_fail}
\end{figure}

These observations highlight how differences in internal trajectory representations, collision models, and execution assumptions can influence outcomes when evaluated under a common, rigorous reference pipeline. While sampling-based planners rely on post-processing for smoothness and discrete optimizers depend on knot-point feasibility, the proposed NFG method operates directly within the function space. By avoiding auxiliary modeling assumptions—such as convexified collision geometry or separate smoothing stages—our approach naturally aligns with the unified execution model, demonstrating intrinsic robustness in constrained environments.
\end{document}